\documentclass[11pt]{article}
\usepackage[letterpaper,margin=1in]{geometry}
\pdfoutput=1
\usepackage{algorithm, algorithmic}
\usepackage{amsmath, amssymb, amsthm}
\usepackage{framed}
\usepackage{verbatim}
\usepackage{hyperref}
\usepackage{color}
\usepackage{xfrac}
\usepackage{todonotes}

\usepackage{natbib}
\usepackage{mark}
\usepackage{marco}
\usepackage{thmtools, thm-restate}

\newif\ifverbose

\verbosetrue

\title{Efficient, Noise-Tolerant, and Private Learning via Boosting}
\author{Mark Bun\thanks{Department of Computer Science,  Boston University. \texttt{mbun@bu.edu}.}
  \and Marco Leandro Carmosino\thanks{School of Computing Science, Simon Fraser University. Supported by PIMS Postdoctoral Fellowship. \texttt{marco@ntime.org}}
  \and Jessica Sorrell\thanks{Department of Computer Science and Engineering,
 University of California, San Diego. \texttt{jlsorrel@ucsd.edu}}}
\begin{document}
\maketitle

\begin{abstract}
  We introduce a simple framework for designing private boosting
algorithms. We give natural conditions under which these algorithms
are differentially private, efficient, and noise-tolerant PAC
learners. To demonstrate our framework, we use it to construct
noise-tolerant and private PAC learners for large-margin halfspaces
whose sample complexity does not depend on the dimension.

We give two sample complexity bounds for our large-margin halfspace
learner. One bound is based only on differential privacy, and uses
this guarantee as an asset for ensuring generalization. This first
bound illustrates a general methodology for obtaining PAC learners
from privacy, which may be of independent interest. The second bound
uses standard techniques from the theory of large-margin
classification (the fat-shattering dimension) to match the best known
sample complexity for differentially private learning of large-margin
halfspaces, while additionally tolerating random label noise.

 \end{abstract}

\clearpage

\tableofcontents

\clearpage

\section{Introduction}
\label{sec:introduction}

\subsection{(Smooth) Boosting, Noise-Tolerance, and Differential Privacy}

Boosting is a fundamental technique in both the theory and practice of machine learning for converting weak learning algorithms into strong ones. Given a sample $S$ of $n$ labeled examples drawn i.i.d. from an unknown distribution, a weak learner is guaranteed to produce a hypothesis that can predict the labels of fresh examples with a noticeable advantage over random guessing. The goal of a boosting algorithm is to convert this weak learner into a strong learner: one which produces a hypothesis with classification error close to zero.

A typical boosting algorithm --- e.g., the AdaBoost algorithm of \citet{FreundS97} --- operates as follows. In each of rounds $t = 1, \dots, T$, the boosting algorithm selects a distribution $D_t$ over $S$ and runs the weak learner on $S$ weighted by $D_t$, producing a hypothesis $h_t$. The history of hypotheses $h_1, \dots, h_t$ is used to select the next distribution $D_{t+1}$ according to some update rule (e.g., the multiplicative weights update rule in AdaBoost). The algorithm terminates either after a fixed number of rounds $T$, or when a weighted majority of the hypotheses $h_1, \dots, h_T$ is determined to have sufficiently low error.

In many situations, it is desirable for the distributions $D_t$ to be
\emph{smooth} in the sense that they do not assign too much weight to
any given example, and hence do not deviate too significantly from the
uniform distribution. This property is crucial in applications of
boosting to noise-tolerant learning~\citep{DomingoW00, Ser01},
differentially private learning~\citep{DworkRV10}, and constructions
of hard-core sets in complexity theory~\citep{Impagliazzo95,
  BarakHK09}. Toward the first of these applications, \citet{Ser01}
designed a smooth boosting algorithm (\texttt{SmoothBoost}) suitable
for PAC learning in spite of malicious noise. In this model of
learning, up to an $\eta$ fraction of the sample $S$ could be
corrupted in an adversarial fashion before being presented to the
learner \citep{Valiant85}. Smooth boosting enables a weak
noise-tolerant learner to be converted into a strong noise-tolerant
learner. Intuitively, the smoothness property is necessary to prevent
the weight placed on corrupted examples in $S$ from exceeding the
noise-tolerance of the weak learner. The round complexity of smooth
boosting was improved by \citet{BarakHK09} to match that of the
AdaBoost algorithm by combining the multiplicative weights update rule
with Bregman projections onto the space of smooth distributions.

Smoothness is also essential in the design of boosting algorithms
which guarantee \emph{differential privacy} \citep{DworkMNS06}, a
mathematical definition of privacy for statistical data
analysis. \citet{KasiviswanathanLNRS11} began the systematic study of
PAC learning with differential privacy. Informally, a (randomized)
learning algorithm is differentially private if the distribution on
hypotheses it produces does not depend too much on any one of its
input samples. Again, it is natural to design ``private boosting''
algorithms which transform differentially private weak learners into
differentially private strong learners. In this context, smoothness is
important for ensuring that each weighted input sample does not have
too much of an effect on the outcomes of any of the runs of the weak
learner. A private smooth boosting algorithm was constructed by
\citet{DworkRV10}, who augmented the AdaBoost algorithm with a private
weight-capping scheme which can be viewed as a Bregman projection.

\subsection{Our Contributions}

\paragraph{Simple and Modular Private Boosting.} Our main result is a
framework for private boosting which simplifies and generalizes the
private boosting algorithm of \citet{DworkRV10}. We obtain
these simplifications by sidestepping a technical issue confronted by
\citet{DworkRV10}. Their algorithm maintains two elements of state
from round to round: the history of hypotheses $H = h_1, \dots, h_t$
and auxiliary information regarding each previous distribution
$D_1, \dots, D_t$, which is used to enforce smoothness. They remark:
``[this algorithm] raises the possibility that adapting an existing or
future smooth boosting algorithm to preserve privacy might yield a
simpler algorithm.''

We realize exactly this possibility by observing that most smooth
boosting algorithms have effectively \emph{stateless} strategies for
re-weighting examples at each round. By definition, a boosting
algorithm must maintain some history of hypotheses. Therefore,
re-weighting strategies that can be computed using only the list of
hypotheses require no auxiliary information. Happily, most smooth
boosting algorithms define such hypothesis-only re-weighting
strategies. Eliminating auxiliary state greatly simplifies our
analysis, and implies natural conditions under which existing smooth
boosting algorithms could be easily privatized.

Our main algorithm is derived from that of \citet{BarakHK09}, which we
call \texttt{BregBoost}. Their algorithm alternates between
mutiplicative re-weighting and Bregman projection: the multiplicative
update reflects current performance of the learner, and the Bregman
projection ensures that \texttt{BregBoost} is smooth. Unfortunately, a
na\"{i}ve translation of \texttt{BregBoost} into our framework would
Bregman project more than once per round. This maintains correctness,
but ruins privacy. Inspired by the private optimization algorithms of
\citet{HsuRU12, HsuRRU14} we give an alternative analysis of
\texttt{BregBoost} that requires only a \emph{single} Bregman
projection at each round. The need for ``lazy'' Bregman projections
emerges naturally by applying our template for private boosting to
\texttt{BregBoost}, and results in a private boosting algorithm with
optimal round complexity: \texttt{LazyBregBoost}. This method of lazy
projection \citep[see][for an exposition]{Rakhlin} has appeared in
prior works about differential privacy \citep{HsuRU12, HsuRRU14}, but
not in the context of designing boosting algorithms.

\paragraph{Application: Privately Learning Large-Margin Halfspaces.}

A halfspace is a function $f : \mathbb{R}^d \to \{-1, 1\}$ of the form
$f(x) = \operatorname{sign}(u \cdot x)$ for some vector
$u \in \mathbb{R}^d$. Given a distribution $D$ over the unit ball in
$\mathbb{R}^d$, the \emph{margin} of $f$ with respect to $D$ is the
infimum of $|u \cdot x|$ over all $x$ in the support of $D$. Learning
large-margin halfspaces is one of the central problems of learning
theory. A classic solution is given by the Perceptron algorithm, which
is able to learn a $\tau$-margin halfspace to classification error
$\alpha$ using sample complexity $O(1/\tau^2\alpha)$ independent of
the dimension $d$. Despite the basic nature of this problem, it was
only in very recent work of \citet{NguyenUZ19} that
dimension-independent sample complexity bounds were given for
\emph{privately} learning large-margin halfspaces. In that work, they
designed a learning algorithm achieving sample complexity
$\tilde{O}(1/\tau^2\alpha\varepsilon)$ for $\tau$-margin halfspaces
with $(\varepsilon, 0)$-differential privacy, and a computationally
efficient learner with this sample complexity for
$(\varepsilon, \delta)$-differential privacy. Both of their algorithms
use dimensionality reduction (i.e., the Johnson-Lindenstrauss lemma)
to reduce the dimension of the data from $d$ to $O(1/\tau^2)$. One can
then learn a halfspace by privately minimizing the hinge-loss on this
lower dimensional space.

Meanwhile, one of the first applications of smooth boosting was to the
study of noise-tolerant learning of halfspaces. \citet{Ser01} showed
that smooth boosting can be used to design a (non-private) algorithm
with sample complexity $\tilde{O}(1/\tau^2\alpha^2)$ which, moreover,
tolerates an $\eta = O(\tau\alpha)$ rate of malicious noise. Given the
close connection between smooth boosting and differential privacy, it
is natural to ask whether private boosting can also be used to design
a learner for large-margin halfspaces. Note that while one could pair
the private boosting algorithm of Dwork, Rothblum, and Vadhan with our
differentially private weak learner for this application, the
resulting hypothesis would be a majority of halfspaces, rather than a
single halfspace. Like \citet{NguyenUZ19} we address \emph{proper}
differentially-private learning of large-margin halfspaces where the
hypothesis is itself a halfspace and not some more complex Boolean
device.

We use our framework for private boosting to achieve a proper
halfspace learner with sample complexity
$\tilde{O}\left(\frac{1}{\epsilon\alpha\tau^2}\right)$ when
$(\epsilon, \delta)$-DP is required. Our learner is simple, efficient,
and automatically tolerates random classification noise
\citep{DBLP:journals/ml/AngluinL87} at a rate of $O(\alpha\tau)$. That
is, we recover the sample complexity of \citet{NguyenUZ19} using a
different algorithmic approach while also tolerating
noise. Additionally, our efficient algorithm guarantees
\emph{zero-concentrated} differential privacy \citep{BS16}, a stronger
notion than $(\epsilon, \delta)$-DP. In this short paper we phrase all
guarantees as $(\epsilon, \delta)$-DP to facilitate comparison of
sample bounds.

\begin{theorem}[Informal, Fat-Shattering Application to Large-Margin
  Halfspaces]
  \label{ithm:fat-shattering-HS-learn}
  Given \\$n = \tilde{O}\left(\frac{1}{\epsilon\alpha\tau^2}\right)$
  samples from a distribution $D$ supported by a $\tau$-margin
  halfspace $u$ subject to $O(\alpha\tau)$-rate random label noise,
  our learning algorithm is $(\epsilon, \delta)$-DP and outputs with
  probability $(1 - \beta)$ a halfspace that $\alpha$-approximates $u$
  over $D$.
\end{theorem}

Furthermore, it may be interesting that we can also obtain non-trivial
sample bounds for the same problem using \emph{only} differential
privacy. The analysis of \citet{NguyenUZ19} uses the VC dimension of
halfspaces and the analyses of \citet{Ser01} and Theorem
\ref{ithm:fat-shattering-HS-learn} above both use the fat-shattering
dimension of halfspaces to ensure generalization. We can instead use
the generalization properties of differential privacy to prove the
following (in Appendix \ref{sec:proof-noise-tolerant}).

\begin{theorem}[Informal, Privacy-Only Application to Large-Margin
  Halfspaces]
  \label{ithm:privacy-only-HS-learn}
  Given\\
  $n = \tilde{O}\left(\frac{1}{\epsilon\alpha\tau^2} + \frac{1}{\alpha^2\tau^2} + \epsilon^{-2} +
    \alpha^{-2}\right)$ samples from a distribution $D$ supported by a
  $\tau$-margin halfspace $u$ subject to $O(\alpha\tau)$-rate random
  label noise, our learning algorithm is $(\epsilon, \delta)$-DP and
  outputs with probability $(1 - \beta)$ a halfspace that
  $\alpha$-approximates $u$ over $D$.
\end{theorem}

Intuitively, the fat-shattering argument has additional
``information'' about the hypothesis class and so can prove better
bounds. However, the argument based only on differential privacy would
apply to \emph{any} hypothesis class with a differentially private
weak learner. So, we present a template for generalization of boosting
in Sections \ref{sec:templ-priv-gener} and \ref{sec:templ-noise-toler}
which relies \emph{only} on the learner's privacy guarantees.

\section{Preliminaries}
\label{sec:preliminaries}

\subsection{Measures \& Distributions}
\label{sec:measures}

For a finite set $X$, let $\cU(X)$ be the uniform distribution over
$X$.

\begin{definition}[Bounded Measures]
  A \emph{bounded measure} on domain $X$ is a function
  $\mu : X \to [0,1]$.
  \begin{quote}
  \begin{description}
  \item[Density:] 
    \(
      d(\mu) = \Ex{x \sim \cU(X)}{\mu(x)}
    \) --- the ``relative size'' of a measure in $X$.
  \item[Absolute Size:]
    \(
    |\mu| = \sum_{x \in X}\mu(x)
    \)
  \item[Induced Distribution:]
    \(
    \dst{\mu}(x) = \mu(x)/|\mu|
    \)  --- the distribution obtained by normalizing a measure
  \end{description}
\end{quote}
\end{definition}

We require some notions of similarity between measures and
distributions.

\begin{definition}[Kullback-Leibler Divergence]
  Let $\mu_1$ and $\mu_2$ be bounded measures over the same
  domain $X$. The \emph{Kullback-Leibler divergence} between $\mu_1$ and
  $\mu_2$ is defined as:
  \[
    \KL{\mu_1}{\mu_2} =
    \sum_{x \in X} \mu_1(x)\log \left( \frac{\mu_1(x)}{\mu_2(x)} \right)
    + \mu_2(x) - \mu_1(x)
  \]
\end{definition}

\begin{definition}[Statistical Distance]
  The \emph{statistical distance} between two distributions $Y$ and
  $Z$, denoted $\Delta(Y,Z)$, is defined as:
  \[
    \Delta(Y,Z) = \max_{S} | \Pr[Y \in S] - \Pr[Z \in S] | 
  \]
\end{definition}

The $\alpha$-R\'{e}nyi divergence has a parameter
$\alpha \in (1, \infty)$ which allows it to interpolate between
KL-divergence at $\alpha = 1$ and max-divergence at $\alpha =
\infty$. 

\begin{definition}[R\'{e}nyi Divergence]
  Let $P$ and $Q$ be probability distributions on $\Omega$. For $\alpha \in (1, \infty)$, we define the R\'{e}nyi Divergence of order $\alpha$ between $P$ and $Q$ as:
\[
  \renD{P}{Q} = \frac{1}{\alpha - 1} \log \left( \Ex{x \sim Q}{\left( \frac{P(x)}{Q(x)}\right)^\alpha}\right) \\
\]

\end{definition}

A measure is ``nice'' if it is simple and efficient to sample from the
associated distribution. If a measure has high enough density, then
rejection sampling will be efficient. So, the set of \emph{high
  density} measures is important, and will be denoted by:

\[
  \Gamma_\kappa = \{ \mu ~|~ d(\mu) \geq \kappa \}.
\]

To maintain the invariant that we only call weak learners on measures
of high density, we use Bregman projections onto the space of high
density measures.

\begin{definition}[Bregman Projection]
  Let $\Gamma \subseteq \R^{|S|}$ be a non-empty closed convex set of
  measures over $S$. The \emph{Bregman projection} of $\tilde{\mu}$
  onto $\Gamma$ is defined as:
	\[
          \Pi_\Gamma \tilde{\mu} =
          \arg\min_{\mu \in \Gamma} \KL{\mu}{\tilde{\mu}}
	\]
\end{definition}

Bregman projections have the following desirable property:

\begin{theorem}[Bregman, 1967]\label{thm:breg}
  Let $\tilde{\mu}, \mu$ be measures such that $\mu \in \Gamma$. Then,
  \[
    \KL{\mu}{\Pi_{\Gamma}\tilde{\mu}}
    +
    \KL{\Pi_{\Gamma} \tilde{\mu}}{\tilde{\mu}}
    \leq \KL{\mu}{\tilde{\mu}}. \text{ In particular, }
        \KL{\mu}{\Pi_{\Gamma}\tilde{\mu}} \leq \KL{\mu}{\tilde{\mu}}.
  \]
\end{theorem}

Barak, Hardt, and Kale gave the following characterization of Bregman
projections onto the set of $\kappa$-dense measures, which we will
also find useful.

\begin{lemma}[Bregman Projection onto $\Gamma_{\kappa}$ is Capped Scaling \citet{BarakHK09}]\label{lem:bregcharacter}
	Let $\Gamma$ denote the set of $\kappa$-dense measures. Let $\unconstmu$ be a measure such that $|\unconstmu | < \kappa n$, and let $c \geq 1$ be the smallest constant such that the measure $\mu$, where $\mu(i) = \min\{1, c\cdot \unconstmu(i)\}$, has density $\kappa$. Then $\Pi_{\Gamma} \unconstmu = \mu$.
\end{lemma}
\subsection{Learning}
\label{sec:learning}

We work in the \textbf{P}robably \textbf{A}pproximately
\textbf{C}orrect (\textbf{PAC}) setting
\citep{DBLP:journals/cacm/Valiant84}. Our learning algorithms
\emph{probably} learn a hypothesis which \emph{approximately} agrees
with an unknown target concept.
We denote by $\cX$
the domain of examples, and for the remainder of this work consider
only the Boolean classification setting where labels are always
$\pm 1$.

\begin{definition}[PAC Learning]\label{def:pac-learn}
  A hypothesis class $\cH$ is $(\alpha,\beta)$\emph{-PAC learnable} if
  there exists a sample bound $n_{\cH} \from (0,1)^2 \to \N$ and a
  learning algorithm $\cA$ such that: for every
  $\alpha, \beta \in (0,1)$ and for every distribution $D$ over
  $\cX \times \{\pm 1\}$, running $\cA$ on
  $n \geq n_{\cH}(\alpha, \beta)$ i.i.d. samples from $D$ will with
  probability at least $(1 - \beta)$ return a hypothesis
  $h : \cX \to \{ \pm 1\}$ such that:
  \[
    \Pr_{(x,y) \sim D} \left[ h(x) = y \right] \geq 1 - \alpha .
  \]
\end{definition}

PAC learners guarantee strong generalization to unseen examples. We
will construct PAC learners by boosting weak learners --- which need
only beat random guessing on any distribution over the training set.

\begin{definition}[Weak Learning]\label{def:weaklearn}
  Let $S \subset (\cX \times \{\pm 1\})^n$ be a training set of size
  $n$. Let $D$ be a distribution over $[n]$. A \emph{weak learning
    algorithm} with \emph{advantage} $\gamma$ takes $(S,D)$ as input
  and outputs a function $h \from \cX \to [-1,1]$ such that:
  \[
    \frac{1}{2}\sum_{j=1}^{n}D(j)|h(x_j) - y_j| \leq \frac{1}{2} - \gamma
  \]
\end{definition}

\subsection{Privacy}

Two datasets $S,S' \in X^n$ are said to be \emph{neighboring} (denoted
$S \sim S'$) if they differ by at most a single element. Differential
privacy requires that analyses performed on neighboring datasets have
``similar'' outcomes. Intuitively, the presence or absence of a single
individual in the dataset should not impact a differentially private
analysis ``too much.''  We formalize this below.

\begin{definition}[Differential Privacy]
  A randomized algorithm $\cM : X^n \to \cR$ is
  $(\eps, \delta)$-differentially private if for all measurable
  $T \subseteq \cR$ and all neighboring datasets $S \sim S' \in X^n$,
  we have
\[\Pr[\cM(S) \in T] \le e^\eps \Pr[\cM(S') \in T] + \delta.\]
\end{definition}

In our analyses, it will actually be more useful to work with the
notion of (zero-)concentrated differential privacy, which bounds
higher moments of privacy loss than normal differential privacy.

\begin{definition}[Zero Concentrated Differential Privacy (zCDP)]
  A randomized algorithm $\cM : X^n \to \cR$ satisfies $\rho$-zCDP if
  for all neighboring datasets $S \sim S' \in X^n$ and all
  $\alpha > 1$, we have $\renD*{\cM(S)}{\cM(S')} \leq \rho \alpha$, where
  $\renD{\cdot}{\cdot}$ denotes the R\'{e}nyi divergence of order
  $\alpha$.
\end{definition}

This second notion will often be more convenient to work with, because it tightly captures the privacy guarantee of Gaussian noise addition and of composition:

\begin{lemma}[Tight Composition for zCDP, \citet{BS16}]\label{lem:zcdp-compose}
If $\cM_1 : X^n \to \cR_1$ satisfies $\rho_1$-zCDP, and $\cM_2 : (X^n \times \cR_1) \to \cR_2$ satisfies $\rho_2$-zCDP, then the composition $\cM : X^n \to \cR_2$ defined by $\cM(S) = \cM_2(S, \cM_1(S))$ satisfies $(\rho_1 + \rho_2)$-zCDP.
\end{lemma}

zCDP can be converted into a guarantee of
$(\eps, \delta)$-differential privacy.

\begin{lemma}[zCDP $\implies$ DP, \citet{BS16}]
  \label{lem:zcdp-to-adp}
Let $\cM: X^n \to \cR$ satisfy $\rho$-zCDP. Then for every $\delta > 0$, we have that $\cM$ also satisfies $(\eps, \delta)$-differential privacy for $\eps = \rho + 2\sqrt{\rho \log(1/\delta)}$.
\end{lemma}

The following lemma will let us bound R\'{e}nyi divergence between
related Gaussians:

\begin{lemma}[R\'{e}nyi Divergence Between Spherical Gaussians;
  folklore, see \citet{BS16}]\label{lem:gaussian-mech}
  Let $z, z' \in \mathbb{R}^d$, $\sigma \in \mathbb{R}$, and
  $\alpha \in [1,\infty).$ Then
	\[\renD{\mathcal{N}(z, \sigma^2 I_d)}
          {\mathcal{N}(z', \sigma^2 I_d)} = \frac{\alpha \| z -
            z'\|_2^2}{2\sigma^2}.\]
\end{lemma}

Finally, $\rho$-zCDP is closed under post-processing, just like standard DP.

\begin{lemma}[Post-processing zCDP, \citet{BS16}]\label{lem:post-processing}
  Let $\cM: X^n \rightarrow R_1$ and $f: R_1 \rightarrow R_2 $ be
  randomized algorithms. Suppose $\cM$ satisfies $\rho$-zCDP. Define
  $\cM': X^n \rightarrow R_2$ by $\cM'(x) = f(M(x))$. Then $\cM'$
  satisfies $\rho$-zCDP.
\end{lemma}

\subsection{Generalization and Differential Privacy}

We now state the generalization properties of differentially private
algorithms that select statistical queries, which count the fraction
of examples satisfying a predicate. Let $X$ be an underlying
population, and denote by $D$ a distribution over $X$.

\begin{definition}[Statistical Queries]
  A statistical query $q$ asks for the expectation of some function on
  random draws from the underlying population. More formally, let
  $q : X \to [0,1]$ and then define the statistical query
  \emph{based on} $q$ (abusing notation) as the following, on a sample
  $S \in X^n$ and the population, respectively:
  \begin{align*}
    q(S) = \frac{1}{|S|}\sum_{x \in S} q(x) && \text{ and } &&
    q(D) = \Ex{x \sim D}{q(x)} 
  \end{align*}
\end{definition}

In the case of statistical queries, dependence of accuracy on the
sample size is good enough to obtain interesting sample complexity
bounds from privacy alone. These transfer theorems have recently been
improved, bringing ``differential privacy implies generalization''
closer to practical utility by decreasing the very large constants
from prior work \citep{DBLP:journals/corr/abs-1909-03577}. As our
setting is asymptotic, we employ a very convienent (earlier) transfer
lemma of \citet{BNSSSU16}.

\begin{theorem}[Privacy $\implies$ Generalization of Statistical
  Queries, \citet{BNSSSU16}]\label{thm:dp-generalization}
  Let $0 < \eps < 1/3$, let $0 < \delta < \eps/4$ and let
  $n \ge \log(4\eps / \delta) / \eps^2$. Let $\cM : X^n \to Q$ be
  $(\eps, \delta)$-differentially private, where $Q$ is the set of
  statistical queries $q : X \to \R$. Let $D$ be a distribution over
  $X$, let $S \getsr D^n$ be an i.i.d. sample of $n$ draws from $D$,
  and let $q \getsr \cM(S)$. Then:
  \[
    \Pr_{q,S} \left[ |q(S) - q(D)| \ge 18 \eps \right] \le \delta / \eps.
  \]
\end{theorem}

\section{Abstract Boosting}
\label{sec:boosting-schemes}

In Section \ref{sec:ensur-priv-boost}, we give sufficient conditions
for private boosting, using a natural decomposition that applies to
many boosting algorithms. In Sections \ref{sec:templ-priv-gener} and
\ref{sec:templ-noise-toler} we use the decomposition to give templates
of sample complexity bounds and noise-tolerant generalization
guarantees for private boosted classifiers that use only the
algorithmic stability imposed by differential privacy. We instantiate
those templates in Appendix \ref{sec:proof-noise-tolerant}, to
construct a noise-tolerant PAC learner for large-margin halfspaces. 

\subsection{Boosting Schemas}
\label{sec:boosting-schemes-1}

A boosting algorithm repeatedly calls a \emph{weak} learning
algorithm, aggregating the results to produce a final hypothesis that
has good training error. Each call to the weak learner re-weights the
samples so that samples predicted poorly by the hypothesis collection
so far are given more ``attention'' (probability mass) by the weak
learner in subsequent rounds. Thus boosting naturally decomposes into
two algorithmic parts: the weak learner \texttt{WkL} and the
re-weighting strategy \texttt{NxM}.

Below, we describe boosting formally using a ``helper function'' to
iterate weak learning and re-weighting. Crucially, we avoid iterating
over any information regarding the intermediate weights; the entire
state of our schema is a list of hypotheses. This makes it easy to
apply a privacy-composition theorem to any boosting algorithm where
\texttt{NxM} and \texttt{WkL} satisfy certain minimal conditions,
elaborated later. Much of the complexity in the analysis of private
boosting by \citet{DworkRV10} was due to carefully privatizing
auxiliary information about sample weights; we avoid that issue
entirely. So, many smooth boosting algorithms could be easily adapted
to our framework.

We denote by $\cH$ the hypotheses used by the weak learner, by $S$ an
i.i.d. sample from the target distribution $D$, by $T$ the number of
rounds, by $\fM$ the set of bounded measures over $S$, and by $\cD(S)$
the set of distributions over $S$.

\begin{minipage}[t]{0.45\linewidth}
  \vspace{-10pt}  
  \begin{algorithm}[H]
    \label{alg:boost}
    \caption{\texttt{Boost}. In:
      $S \in X^n,~ T \in \mathbb{N}$}
    \begin{algorithmic}
      \STATE $H \gets \{\}$
      \FOR{$t = 1$ to $T$}
      \STATE $H \gets \mathtt{Iter}(S, H)$
      \ENDFOR
      \STATE $\hat{f}(x) \gets \frac{1}{T}\sum_{i=1}^{T}h_i(x)$
      \RETURN $\sgn(\hat{f}(x))$
    \end{algorithmic}
  \end{algorithm}
\end{minipage}
\begin{minipage}[t]{0.45\linewidth}
  \vspace{-10pt}
  \begin{algorithm}[H]
    \caption{\texttt{Iter}. In: $S \in X^n, ~ H \in \cH^*$}
    \begin{algorithmic}
      \STATE $\mu \gets \mathtt{NxM}(S, H)$
      \STATE $h \gets \mathtt{WkL}(S, \mu)$
      \RETURN $H \cup \{h\}$
      \STATE \texttt{//} Add $h$ to list of hypotheses
      \vspace{10pt}
    \end{algorithmic}
  \end{algorithm}
\end{minipage}

\subsection{Ensuring Private Boosting}
\label{sec:ensur-priv-boost}

Under what circumstances will boosting algorithms using this schema
guarantee differential privacy? Since we output (with minimal
post-processing) a collection of hypotheses from the weak learner, it
should at least be the case that the weak learning algorithm
\texttt{WkL} is itself differentially private. In fact, we will need
that the output distribution on hypotheses of a truly private weak
learner does not vary too much if it is called with both similar
\emph{samples} and similar \emph{distributional targets}.

\begin{definition}[Private Weak Learning]
  \label{def:private-wkl}
  A weak learning algorithm $\mathtt{WkL} : S \times \cD(S) \to \cH$
  satisfies $(\rho, s)$-zCDP if for all neighboring samples
  $S \sim S' \in (\cX^n \times \{ \pm 1\})$, all $\alpha > 1$, and any
  pair of distributions $\dst{\mu}, \dst{\mu}'$ on $X$ such that
  $\Delta(\dst{\mu},\dst{\mu}') < s$, we have:
  \[
    \renD{\mathtt{WkL}(S,\dst{\mu})}{\mathtt{WkL}(S',\dst{\mu}')}
    \le \rho \alpha
  \]
\end{definition}

This makes it natural to demand that neighboring samples induce
similar measures. Formally:

\begin{definition}[$\zeta$-Slick Measure Production]
  \label{def:slick-measures}
  A measure production algorithm $\mathtt{NxM} : S \times \cH \to \fM$
  is called $\zeta$-slick if, for all neighboring samples
  $S \sim S' \in (\cX^n \times \{ \pm 1\})$ and for all sequences of
  hypotheses $H \in \cH^*$, letting $\dst{\mu}$ and $\dst{\mu}'$ be
  the distributions induced by $\mathtt{NxM}(S,H)$ and
  $\mathtt{NxM}(S',H)$ respectively, we have:
  \[
    \Delta(\dst{\mu} , \dst{\mu}') \leq \zeta
  \]
\end{definition}

It is immediate that a single run of \texttt{Iter} is private if it
uses \texttt{NxM} and \texttt{WkL} procedures that are appropriately
slick and private, respectively. Suppose \texttt{WkL} is
$(\rho_W, \zeta)$-zCDP and \texttt{NxM} is $\zeta$-slick. By
composition, \texttt{Iter} run using these procedures is
$\rho_W$-zCDP.  Finally, observe that \texttt{Boost} paired with a
private weak learner and slick measure production is $T\rho_W$-zCDP,
because the algorithm simply composes $T$ calls to \texttt{Iter} and
then post-processes the result.

\subsection{Template: Privacy $\implies$ Boosting Generalizes}
\label{sec:templ-priv-gener}
We now outline how to use the generalization properties of
differential privacy to obtain a PAC learner from a private weak
learner, via boosting. Recall that the fundamental boosting theorem is
a \emph{round bound:} after a certain number of rounds, boosting
produces a collection of weak hypotheses that can be aggregated to
predict the \emph{training data} well.

\begin{theorem}[Template for a Boosting Theorem]
  Fix \texttt{NxM}. For any weak learning algorithm \texttt{WkL} with
  advantage $\gamma$, running \texttt{Boost} using these concrete
  subroutines terminates in at most $T(\gamma, \alpha, \beta)$ steps
  and outputs (with probability at least $1 - \beta$) a hypothesis $H$
  such that:
  \[
    \Pr_{(x,y) \sim S} \lbrack H(x) \neq y \rbrack \leq \alpha
  \]
\end{theorem}

We can capture the training error of a learning algorithm using a
statistical query. For any hypothesis $H$, define:
\[
 \err_H(x, y) \mapsto \begin{cases}
   1 &\text{if } H(x) \neq y \\
   0 &\text{otherwise}
 \end{cases}
\]

Denoting by $D$ the target distribution, PAC learning demands a
hypothesis such that $\err_H(D) \leq \alpha$, with high
probability. If the boosting process is differentially private, the
generalization properties of differential privacy ensure that
$\err_H(S)$ and $\err_H(D)$ are very close with high
probability. Thus, boosting private weak learners can enforce low test
error. We elaborate below.

\begin{theorem}[Abstract Generalization]
  Let \texttt{WkL} be a $(\rho, \zeta)$-zCDP weak learner with
  advantage $\gamma$. Suppose \texttt{NxM} is $\zeta$-slick and enjoys
  round-bound $T$ with error $\alpha$ and  failure probability $\beta$. Denote by $\cM'$ the algorithm \texttt{Boost} run
  using \texttt{WkL} and \texttt{NxM}. Let
  $\epsilon = O(\sqrt{\rho T \log(1/\delta)})$ and suppose
  $n \geq \Omega( \log(\epsilon/\delta)/\epsilon^2)$. Then, with probability at least $1 - \beta - \delta/\epsilon$ over
  $S \sim_{iid} D^n$ and the internal randomness of $\cM'$, the
  hypothesis output by $\cM'$ \emph{generalizes} to $D$:
  \[
    \Pr_{(x,y) \sim D} \lbrack H(x) \neq y \rbrack \leq \alpha.
  \]
\end{theorem}

\begin{proof}[Proof (sketch).]
  By the round-bound and inspection of Algorithm \ref{alg:boost},
  $\cM'$ simply composes $T$ calls to \texttt{Iter} and
  post-processes. So by zCDP composition (Lemma
  \ref{lem:zcdp-compose}) we know that $\cM'$ is $\rho
  T$-zCDP. This can be converted to an $(\epsilon, \delta)$-DP guarantee on
  $\cM'$ for any $\delta > 0$ (Lemma \ref{lem:zcdp-to-adp}).

  For the sake of analysis, define a new mechanism $\cM$ that runs
  $\cM'(S)$ to obtain $H$ and then outputs the statistical query
  $\err_H$. This is just post-processing, so $\cM$ is also
  $(\epsilon, \delta)$-DP. Thus, given enough samples, the conversion
  of privacy into generalization for statistical queries applies to
  $\cM$:

\[
      \Pr_{S \sim D^n}[ | \err_H(S) - \err_H(D) | \ge 18 \eps] \le \delta / \eps
      \text{  (Theorem \ref{thm:dp-generalization}) }.\]
By the guarantee of the round-bound, $\err_H(S) \leq \alpha$ with probability at least $1-\beta$. Therefore,
     \[\Pr_{S \sim D^n} \left[ \Pr_{(x,y) \sim D}
        \left[ H(x) \neq y \right] \leq \alpha + 18 \eps
      \right] \leq \delta / \epsilon + \beta.\]

\end{proof}

Observe that we require privacy both for privacy's sake and for the
generalization theorem. Whichever requirement is more stringent will
dominate the sample complexity of any algorithm so constructed.

\subsection{Template: Privacy $\implies$ Noise-Tolerant Generalization}
\label{sec:templ-noise-toler}

Suppose now that there is some kind of interference between our
learning algorithm and the training examples. For example, this could
be modeled by random classification noise with rate $\eta$
\citep{DBLP:journals/ml/AngluinL87}. This altered setting violates the
preconditions of the DP to generalization transfer. A noised sample is
\emph{not} drawn i.i.d. from $D$ and so the differential privacy of
$\cM$ is not sufficient to guarantee generalization of the ``low
training error'' query $\err_H$ as defined above.

To get around this issue, we fold a noise model into the
generalization-analysis mechanism. Define an alternative \emph{noised}
mechanism $\cM_\eta$ (Algorithm \ref{alg:abstract-noised-boost}) atop
any $\cM$ that outputs a ``test error'' query, and apply ``DP to
Generalization'' on $\cM_\eta$ instead. Suppose that $\cM$ is
differentially private, and the underlying learning algorithm $\cA$
run by $\cM$ tolerates noise at rate $\eta$. Then, if $\cM_\eta$ is
DP, we can generalize the noise-tolerance of $\cA$.

    \begin{algorithm}
      \caption{$\cM_\eta$ for RCN.
        Input: $S \in X^n, S \sim D^n$}
      \label{alg:abstract-noised-boost}
      \begin{algorithmic}
        \STATE $\forall i \in [n] ~ \rv{F}_i \gets 1$
        \STATE $\forall i \in [n] ~ \rv{F}_i \gets -1 $ with probability $\eta$
        \STATE $\tilde{y}_i \gets y_i \rv{F}_i $
        \STATE $\tilde{S} \gets \{(x_i,\tilde{y}_i)\}$
        \STATE $\err_H \gets \cM(\tilde{S})$
        \RETURN $\err_H$
      \end{algorithmic}
    \end{algorithm}

At least for random classification noise, $\cM_\eta$ does indeed
maintain privacy. Observe that for a fixed noise vector $\vec{N}$,
$\cM_\eta$ run on neighboring data sets $S$ and $S'$ will run $\cM$
with neighboring datasets $\tilde{S}$ and $\tilde{S}'$, and therefore
the output distributions over queries will have bounded
distance. Since the noise is determined independent of the sample,
this means that $\cM_\eta$ inherits the differential privacy of $\cM$,
and therefore satisfies the conditions of
Theorem~\ref{thm:dp-generalization}. So the resulting learner still
generalizes.

This trick could handle much harsher noise models. For instance, each
example selected for noise could be arbitrarily corrupted instead of
given a flipped label. But we seem unable to capture fully malicious
noise: an adversary viewing the whole sample could compromise privacy
and so generalization. Thus, the ``effective noise model'' implicit
above seems to distinguish between adversaries who have a global
versus local view of the ``clean'' sample. This seems a natural
division; we hope that future work will explore the expressiveness of
this noise model.

\section{Concrete Boosting via Lazy Bregman Projection}
\label{sec:concrete-boosting}

We instantiate the framework above. This requires a ``Next Measure''
routine (\texttt{LB-NxM}, Algorithm \ref{alg:LB-NxM}) a Boosting
Theorem (Theorem \ref{thm:BregBoost-RoundBound}) and a slickness bound
(Lemma \ref{lem:LB-NxM-slickness}).

\subsection{Measure Production Using Lazy Dense Multiplicative
  Weights}
\label{sec:meas-prod-lazy}

Our re-weighting strategy combines multiplicative weights with Bregman
projections. In each round, we compute the collective margin on each
example. Then, we multiplicative-weight the examples according to
error: examples predicted poorly receive more weight. Finally, to
ensure that no example receives too much weight, we Bregman-project
the resulting measure into the space $\Gamma$ of $\kappa$-dense
measures. We call this strategy ``lazy'' because projection happens
only \emph{once} per round.

\begin{algorithm}
  \caption{\texttt{LB-NxM($\kappa, \lambda$)}: Lazy-Bregman Next Measure}
  \label{alg:LB-NxM}
  \emph{\textbf{Parameters:}} $\kappa \in (0,1)$, desired density of
  output measures; $\lambda \in (0,1)$, learning rate \\
  \emph{\textbf{Input:}} $S$, the sample;
  $H = \{h_1, \dots, h_t\}$, a sequence of hypotheses \\
  \emph{\textbf{Output:}} A measure over $[n]$, $n = |S|$
  \begin{algorithmic}
    \STATE $\mu_1(i) \gets \kappa ~~ \forall i \in [n]$
    \COMMENT{Initial measure is uniformly $\kappa$}
    \FOR{$j \in [t]$}
    \STATE $\ell_j(x_i) \gets 1 - \tfrac{1}{2}|h_j(x_i) - y_i|
    ~~ \forall i \in [n]$
    \COMMENT{Compute error of each hypothesis}
    \ENDFOR
    \STATE $\tilde{\mu}_{t+1}(i) \gets
    e^{-\lambda\sum_{j=1}^t \ell_j(x_i)} \mu_1(i)
    ~~ \forall i \in [n]$
    \STATE $\mu_{t+1} \gets \Pi_{\Gamma}(\tilde{\mu}_{t+1})$
    \STATE return $\dst{\mu}_{t+1}$
  \end{algorithmic}
\end{algorithm}

\texttt{LB-NxM} is typed correctly for substitution into the
\texttt{Boost} algorithm above; the measure is computed using
\emph{only} a sample and current list of hypotheses. Thus,
\texttt{LazyBregBoost} = \texttt{Boost(LB-NxM)} admits a simple
privacy analysis as in Section \ref{sec:ensur-priv-boost}.

\subsection{Boosting Theorem for Lazy Bregman Projection}
\label{sec:boost-theor-lazy}

Given a weak learner that beats random guessing, running
\texttt{LazyBregBoost} yields low training error after a bounded
number of rounds; we prove this in Appendix
\ref{sec:boost-via-games}. Our argument adapts the well-known
reduction from boosting to iterated play of zero-sum games
\citep{FS94} for hypotheses with real-valued outputs. For
completeness, we also give a self-contained analysis of the
iterated-play strategy corresponding to \texttt{LB-NxM} in Appendix
\ref{sec:lazy-regret-bound}. Similar strategies are used by other
differentially-private algorithms \citep{HsuRU12, HsuRRU14} and their
properties are known to follow from results in online convex
optimization \citep{DBLP:journals/ftml/Shalev-Shwartz12}. However, to
our knowledge an explicit proof for the ``lazy'' variant above does
not appear in the literature; so we include one in Appendix
\ref{sec:lazy-regret-bound}. Overall, we have the follwing:

\begin{restatable}[Lazy Bregman Round-Bound]{theorem}{LazyBregRB}
  \label{thm:BregBoost-RoundBound}
  Suppose we run \texttt{Boost} with
  \texttt{LB-NxM($\kappa, \gamma/4$)} on a sample
  $S \subset \mathcal{X}\times \{\pm 1\}$ using any real-valued weak
  learner with advantage $\gamma$ for
  $T \geq \frac{16\log{(1/\kappa)}}{\gamma^2}$ rounds. Let
  $H \from \mathcal{X} \to [-1, 1 ]$ denote the final, aggregated
  hypothesis. The process has:
  \begin{quote}
    \begin{description}
    \item[Good Margin:] $H$ mostly agrees with the
      labels of $S$.
      \[
        \Pr_{(x,y) \sim S} \left[ yH(x) \leq \gamma \right] \leq \kappa
      \]
    \item[Smoothness:] Every distribution $\dst{\mu_t}$ supplied
      to the weak learner has
      $\dst{\mu_t}(i) \leq \frac{1}{\kappa n} ~\forall i$
    \end{description}
  \end{quote}
\end{restatable}

\subsection{Slickness Bound for Lazy Bregman Projection}
\label{sec:slickness-bound-lazy}

\texttt{LB-NxM} is ``lazy'' in the sense that Bregman projection
occurs only once per round, after all the multiplicative updates. The
projection step is \emph{not} interleaved between multiplicative
updates. This is necessary to enforce slickness, which we require for
privacy as outlined in Section \ref{sec:ensur-priv-boost}.

\begin{restatable}[Lazy Bregman Slickness]{lemmma}{LazyBregSlick}
  \label{lem:LB-NxM-slickness}
  The dense measure update rule \texttt{LB-NxM} ~(Algorithm
  ~\ref{alg:LB-NxM}) is $\zeta$-slick for $\zeta = 1/\kappa n$.
\end{restatable}

\begin{proof}[Proof of Lemma \ref{lem:LB-NxM-slickness}]
	Let $\unconstmu, \unconstmu'$ be the unprojected measures produced at the end of the outermost loop of  $\mathtt{NxM}$, when $\mathtt{NxM}$ is run with the sequence of hypotheses $H = \{h_1, \dots, h_T\}$, and on neighboring datasets $S \sim S'$. Let $i$ be the index at which $S$ and $S'$ differ, and note that $\unconstmu(j) = \unconstmu'(j)$ for all $j \neq i$. 
	
	Let $\unconstmu_0$ denote the measure with $\unconstmu_0(j) = \unconstmu(j) = \unconstmu'(j)$ for all $j \ne i$, and $\unconstmu_0(i) = 0$. Take $\Gamma$ to be the space of $\kappa$-dense measures, and let $\mu_0 = \Pi_{\kappa} \unconstmu_0$ and $\mu = \Pi_{\kappa} \unconstmu$ denote the respective projected measures. We will show that $SD(\dst{\mu_0}, \dst{\mu}) \le 1/\kappa n$, which is enough to prove the claim by the triangle inequality. (Note that $|\mu_0| = |\mu| = \kappa n$, which follows from Lemma~\ref{lem:bregcharacter} and the observation that $|\unconstmu_0| \le |\unconstmu| \le \kappa n$. Moreover, $\mu_0(j) \ge \mu(j)$ for every $j \ne i$. )
	
	We calculate
	\begin{align*}
	\sum_{j = 1}^n |\mu_0(j) - \mu(j)| &= |\mu(i)| + \sum_{j \ne i} |\mu_0(j) - \mu(j)| \\
	&\le 1 + \sum_{j \ne i} \mu_0(j) - \mu(j) \\
	&= 1 + |\mu_0| - (|\mu| - \mu(i)) \\
	&\le 1 + |\mu_0| - |\mu| + 1 \\
	&= 2,
	\end{align*}
	since $\mu$ and $\mu_0$ have density $\kappa$. Hence,
	\begin{align*}
	\Delta(\dst{\mu}, \dst{\mu_0}) &= \frac{1}{2} \sum_{i = 1}^n \left| \frac{\mu(i)}{|\mu|} - \frac{\mu_0(i)}{|\mu_0|} \right| \\
	&= \frac{1}{2\kappa n} \sum_{i = 1}^n |\mu(i) - \mu_0(i)| \\
	&\le \frac{1}{\kappa n}.
	\end{align*}
\end{proof}

\section{Application: Learning Halfspaces with a Margin}

\subsection{Learning Settings}
\label{sec:geometric-setup}

We first assume realizability by a large-margin halfspace. Let $u$ be
an unknown unit vector in $\R^d$, and let $D$ be a distribution over
examples from the $\ell_2$ unit ball $B_d(1) \subset \R^d$. Further
suppose that $D$ is $\tau$-good for $u$, meaning
$|u \cdot x| \ge \tau$ for all $x$ in the support of $D$. A PAC
learner is given access to $n$ i.i.d. labeled samples from $D$,
honestly labeled by $u$.

A noise-tolerant learner is given access to a \emph{label noise}
example oracle with noise rate $\eta$, which behaves as follows. With
probability $1- \eta$, the oracle returns a clean example
$(x, \sgn(u \cdot x))$ for $x \sim D$. With probability $\eta$, the
oracle returns an example with the label flipped:
$(x, -\sgn(u \cdot x))$ for $x \sim D$. Given access to the noisy
example oracle, the goal of a leaner is to output a hypothesis
$h : B_d \to \bits$ which $\alpha$-approximates $u$ under $D$, i.e.,
$\Pr_{x \sim D}[h(x) \ne \sgn(u \cdot x)] \le \alpha$
\citep{DBLP:journals/ml/AngluinL87}.

\citet{Ser01} showed that smooth boosting can be used to solve this
learning problem under the (more demanding) \emph{malicious noise}
rate $\eta = O(\alpha \tau)$ using sample complexity
$n = \tilde{O}(1 / (\tau \alpha)^2)$. We apply the Gaussian
mechanism to his weak learner to construct a differentially private
weak learner, and then boost it while preserving privacy. Our (best)
sample complexity bounds then follow by appealing to the
fat-shattering dimension of bounded-norm halfspaces in Section
\ref{sec:gen-vaa-fat-shatt}. Slightly worse bounds proved using only
differential privacy are derived in Appendix
\ref{sec:proof-noise-tolerant}.

\subsection{Weak Halfspace Learner: Centering with Noise}
\label{sec:weak-halfsp-learn}

The noise-tolerant weak learner for halfspaces was
$\weak(S, \dst{\mu})$ which outputs the hypothesis $h(x) = z \cdot x$
where
\[z = \sum_{i  = 1}^n \dst{\mu}(j) \cdot y_i \cdot x_i.\]
The accuracy of this learner is given by the following theorem:

\begin{theorem}[\citet{Ser01}]
Let $\dst{\mu}$ be a distribution over $[n]$ such that $L_\infty(\dst{\mu}) \le 1 / \kappa n$. Suppose that at most $\eta n$ examples in $S$ do not satisfy the condition $y_i \cdot (u \cdot x_i) \ge \tau$ for $\eta \le \kappa \tau / 4$. Then $\weak(S, \dst{\mu})$ described above returns a hypothesis $h : B_d \to [-1 , 1]$ with advantage at least $\tau / 4$ under $\dst{\mu}$.
\end{theorem}

We apply the Gaussian mechanism to Servedio's weak learner to obtain
$\widehat{\weak}(S, \dst{\mu}, \sigma)$ which outputs
$h(x) = \hat{z} \cdot x$ for
\[\hat{z} = \sum_{i  = 1}^n \dst{\mu}(j) \cdot y_i \cdot x_i + \nu,\]
with noise $\nu \sim \cN(0, \sigma^2I_d)$. We get a similar advantage
bound, now trading off with privacy.

\begin{restatable}[Private Weak Halfspace
  Learner]{theorem}{privateHalfspaceWkL}\label{thm:wL-zCDP}
  Let $\dst{\mu}$ be a distribution over $[n]$ such that
  $L_{\infty}(\dst{\mu}) \leq 1/\kappa n $. Suppose that at most
  $\eta n$ examples in $S$ do not satisfy the condition
  $y_i\cdot(u\cdot x_i) \geq \tau$ for $\eta \leq \kappa \tau/4$. Then
  we have:
  \begin{enumerate}
  \item \textbf{Privacy:} $\widehat{\weak}(S, \dst{\mu}, \sigma)$
    satisfies $(\rho,s)$-zCDP for
    $\rho = \frac{2(1/\kappa n + s)^2}{\sigma^2}$.
  \item \textbf{Advantage:} There is a constant $c$ such that for any
    $\xi > 0$, $\widehat{\weak}(S, \dst{\mu}, \sigma)$ returns a
    hypothesis $h : B_d \to [-1 , 1]$ that, with probability at least
    $1 - \xi$, has advantage at least
    $\tau / 4 - c\sigma\sqrt{\log(1/\xi)}$ under $\dst{\mu}$.
  \end{enumerate}
\end{restatable}

\begin{proof}
	We begin with the proof of privacy. 
	Let $\dst{\mu}_1, \dst{\mu}_2$ be $\kappa$-smooth distributions over $[n]$  with statistical distance $\Delta(\dst{\mu}_1, \dst{\mu}_2) \leq s$. Let $S \sim S'$ be neighboring datasets with $\{(x_i, y_i)\} = S\setminus S'$ and $\{(x_i', y_i')\} = S'\setminus S$. Then we have 
	\begin{align*} 
	\|\hat{z}_{S, \dst{\mu}_1} - \hat{z}_{S', \dst{\mu}_2}\|_2 
	&= \| \dst{\mu}_1(i) y_i\cdot x_i - \dst{\mu}_2(i) y_i' \cdot x_i' + \sum_{\substack{j=1\\ j \neq i }}^n (\dst{\mu}_1(j) - \dst{\mu}_2(j))y_j\cdot x_j \|_2  \\ 
	&\leq \| \dst{\mu}_1(i) y_i\cdot x_i \|_2 + \|\dst{\mu}_2(i) y_i' \cdot x_i'\|_2 + \sum_{\substack{j=1\\ j \neq i }}^n \|(\dst{\mu}_1(j) - \dst{\mu}_2(j))y_j\cdot x_j \|_2 \\
	& = \dst{\mu}_1(i) + \dst{\mu}_2(i) + \sum_{\substack{j=1\\ j \neq i }}^n |(\dst{\mu}_1(j) - \dst{\mu}_2(j))| \\
	&\leq 2\dst{\mu}_2(i) +   \sum_{j=1}^n |\dst{\mu}_1(j) - \dst{\mu}_2(j)| \\
	&\leq 2(1/\kappa n + s).
	\end{align*} 
	Then Lemma~\ref{lem:gaussian-mech} gives us that \[\renD*{\widehat{\weak}(S, \dst{\mu}_1, \sigma)}{\widehat{\weak}(S', \dst{\mu}_2, \sigma)} \leq \frac{2\alpha(1/\kappa n + s)^2}{\sigma^2} \] and therefore $\widehat{\weak}(S, \dst{\mu}_1, \sigma)$ satisfies $(\rho, s)$-zCDP for $\rho = \frac{2(1/\kappa n + s)^2}{\sigma^2}$.	
	
	Building on Servedio's result, we now give the advantage lower bound. Servedio's argument shows that the advantage of $\widehat{\weak}(S, \dst{\mu}, \sigma)$ is at least $\hat{z} \cdot u / 2 = z \cdot u / 2 + \nu \cdot u / 2$. Since $\nu$ is a spherical Gaussian and $u$ is a unit vector, we have that for any $\xi > 0$,
	\[\Pr[|\nu \cdot u| \ge c \sigma \sqrt{\log (1/\xi)}] \le \xi.\]
\end{proof}

\subsection{Strong Halfspace Learner: Boosting}

Putting all the pieces together, we run \texttt{Boost} using the
private weak halfspace learner (Theorem \ref{thm:wL-zCDP}) and
lazy-Bregman measures (Theorem \ref{thm:BregBoost-RoundBound}). Via
the composition theorem for differential privacy, we get a privacy
guarantee for the terminal hypothesis as outlined in Section
\ref{sec:ensur-priv-boost}. Finally, we use the fat shattering
dimension to ensure that this hypothesis generalizes.

\begin{algorithm}\label{alg:HS-StL}
  \caption{Strong Halfspace Learner, via Boosting (\texttt{HS-StL})}
  \emph{\textbf{Input:}} Sample: $S$; Parameters:
  $(\alpha,\beta)$-PAC, $(\epsilon, \delta)$-DP, $\tau$-margin \\
  \emph{\textbf{Output:}} A hypothesis $H$
  \begin{algorithmic}
    \STATE $\sigma \gets \tau/8c\sqrt{\log\left(\frac{3072\log(1/\kappa)}{\beta \tau^2}\right)}$\\
    \STATE $ T \gets 1024\log(1/\kappa)/\tau^2$
    \STATE $H \gets \texttt{Boost}$ run with \texttt{LB-NxM}$(\kappa \coloneqq (\alpha/4),
    \lambda \coloneqq (\tau/8))$
    and $\widehat{\weak}(\cdot,\cdot,\sigma)$ for $T$ rounds
  \end{algorithmic}
\end{algorithm}

\subsection{Generalization via fat-shattering dimension.} 
\label{sec:gen-vaa-fat-shatt}

Following the analysis of \citet{Ser01}, we can show that with high probability the hypothesis output by our halfspace learner will generalize, even for a sample drawn from a distribution with random classification noise at rate $O(\alpha\tau)$. The proof of generalization goes by way of fat-shattering dimension. Using an argument nearly identical to that of \citet{Ser00}, we can bound the fat-shattering dimension of our hypothesis class. This bound, along with the guarantee of Theorem~\ref{thm:BregBoost-RoundBound} that our final hypothesis will have good margin on a large fraction of training examples, allows us to apply the following generalization theorem of Bartlett and Shawe-Taylor, which bounds the generalization error of the final hypothesis.

\begin{theorem}[\citet{Bartlett98}]\label{thm:fatshatgen}
	Let $\mathcal{H}$ be a family of real-valued hypotheses over some domain $\mathcal{X}$, let $D$ be a distribution over labeled examples $\mathcal{X}\times\{-1,1\}$. Let $S = \{(x_1,y_1), \dots, (x_n, y_n)\}$ be a sequence of labeled examples drawn from $D$, and let $h(x) = $ sign($H(x))$ for some $H \in \mathcal{H}$. If $h$ has margin less than $\gamma$ on at most $k$ examples in $S$, then with probability at least $1-\delta$ we have that 
	\[\Pr_{(x,y)\sim D}[h(x) \neq y] \leq \frac{k}{n} + \sqrt{\frac{2d}{n}\ln(34en/d)\log(578n) + \ln(4/\delta)}\]
	where $d = fat_F(\gamma/16)$ is the fat-shattering dimension of $\mathcal{H}$ with margin $\gamma/16$. 
\end{theorem}

In order to meaningfully apply the above theorem, we will need to
bound the fat-shattering dimension of our hypothesis class
$\mathcal{H}$. Our bound (proved in Appendix
\ref{sec:fat-shatter-apx}) follows from the analysis of \citet{Ser00},
but given that our hypothesis class is not exactly that analyzed in
\citet{Ser00}, the bound holds only when the noise added to the
hypotheses at each round of boosting does not increase the $\ell_2$
norm of the final hypothesis by too much.

\begin{restatable}[Fuzzy Halfspace Fat Shattering Dimension]{lemmma}{FuzzyFatHS}
  \label{lem:shattdim}
  With probability $1 - \tfrac{\beta}{3}$, after
  $T = \frac{1024\log(1/\kappa)}{\tau^2}$ rounds of boosting,
  Algorithm~\ref{alg:HS-StL} outputs a hypothesis in a class with
  fat-shattering dimension
  $fat_{\mathcal{H}}(\gamma) \leq 4/\gamma^2$.
\end{restatable}

With the above bound on fat-shattering dimension, we may prove the following properties of \texttt{HS-StL}. 
\begin{theorem}[Private Learning of Halfspaces]
	\label{thm:st-hs-learn}
	The \texttt{HS-StL} procedure, is a
	$(\epsilon, \delta)$-Differentially Private $(\alpha, \beta)$-strong
	PAC learner for $\tau$-margin halfspaces, tolerating random classification noise at rate $O(\alpha\tau)$, with sample complexity
	\begin{align*}
	n = \Omega \Bigg(
	\underbrace{
		\frac{\sqrt{\log(1/\alpha)\log(1/\delta)\log(\log(1/\alpha)/\beta\tau^2)}}
		{\epsilon\alpha \tau^2}
	}_{\text{privacy}} &
	+
	\underbrace{
		\frac{ \log(1/\tau\alpha)\log(1/\alpha)}
		{\alpha^2 \tau^2} + \frac{\log(1/\beta)}{\kappa\tau}
	}_{\text{accuracy}}
	\Bigg)
	\end{align*}
\end{theorem}
\begin{proof}
We begin by calculating the cumulative zCDP guarantee of
\texttt{HS-StL}. First, by the privacy bound for $\widehat{\weak}$
(Theorem \ref{thm:wL-zCDP}), we know that a single iteration of
\texttt{Boost} is $\rho$-zCDP for $\rho = \frac{8}{(\kappa n \sigma)^2}$. Furthermore, by tight composition
for zCDP (Lemma \ref{lem:zcdp-compose}) and our setting of $T$, \texttt{HS-StL} is
$\rho_T$-zCDP where:
\[
\rho_T = O \left(
\frac{\log(1/\kappa)}{(\kappa n \sigma \tau)^2}
\right).
\]

Denote by $\epsilon$ and $\delta$ the parameters of approximate
differential privacy at the final round $T$ of \texttt{HS-StL}.  Now
we convert from zero-concentrated to approximate differential
privacy, via Lemma \ref{lem:zcdp-to-adp}: for all $\delta >0$, if
$\epsilon > 3\sqrt{\rho_T\log(1/\delta)}$, then \texttt{HS-StL}
is $(\epsilon, \delta)$-DP. So, for a given target $\epsilon$ and $\delta$, taking
\[n \in O\left(\frac{\sqrt{\log(1/\kappa)\log(1/\delta)\log(\log(1/\kappa)/\beta\tau)}}{\epsilon\kappa\tau^2}\right)\]
will ensure the desired privacy. 

We now turn to bounding the probability of events that could destroy
good training error.

\begin{description}
	\item[Too Many Corrupted Samples.] Our proof of $\widehat{\weak}$'s advantage required that fewer than $\kappa \tau n/4$ examples are corrupted. At noise rate $\eta \leq \kappa \tau/8$, we may use a Chernoff bound to argue that the probability of exceeding this number of corrupted samples is at most $\beta/3$, by taking $n > \frac{24\log (3/\beta)}{\kappa \tau}$. 
	\item[Gaussian Mechanism Destroys Utility.] The Gaussian noise
	injected to ensure privacy could destroy utility for a round of
	boosting. Our setting of $\sigma$ simplifies the advantage of $\widehat{\weak}$ to $\gamma(\tau, \sigma) = \tau/8$ with all but probability $\xi = \frac{\beta \tau^2}{3072\log(1/\kappa)}$. Then we have that with probability $(1 - \xi)^T \geq 1 - \frac{\beta}{3}$, every hypothesis output by $\widehat{\weak}$ satisfies the advantage bound $\gamma \geq \tau/8$. Therefore, by Theorem~\ref{thm:BregBoost-RoundBound}, \texttt{HS-StL} only fails to produce a hypothesis with training error less than $\kappa$ with probability $\beta/3$. 
\end{description}

We now consider events that cause generalization to fail. 

\begin{description}
	\item[Final hypothesis $H \not\in \mathcal{H}$.] The Gaussian noise added to ensure privacy could cause the final hypothesis $H$ to fall outside the class $\mathcal{H} = \{f(x) = z \cdot x : \|z\|_2 \leq 2\}$, for which we have a fat-shattering dimension bound. The probability of this event, however, is already accounted for by the probability that the Gaussian Mechanism destroys the weak learner's utility, as both failures follow from the Gaussian noise exceeding some $\ell_2$ bound. The failures that affect utility are a superset of those that affect the fat-shattering bound, and so the $\beta/3$ probability of the former subsumes the probability of the latter. 
	\item[Failure internal to generalization theorem.] Theorem~\ref{thm:fatshatgen} gives a generalization guarantee that holds only with some probability. We denote the probability of this occurrence by $\beta_1$.
\end{description}

If none of these failures occur, it remains to show that we can achieve accuracy $\alpha$. From Lemma~\ref{lem:roundbound}, we have that $H$ will have margin $\gamma = \tau/8$ on all but a $\kappa$ fraction of the examples, some of which may have been corrupted. We assume the worst case -- that $H$ is correct on all corrupted examples. We have already conditioned on the event that fewer than $\kappa \tau n/4$ examples have been corrupted, and so we may then conclude that $H$ has margin less than $\gamma$ on at most a $2\kappa$ fraction of the uncorrupted examples. Then if we set $\kappa = \alpha/4$ and take
\[n \in O\left(\frac{ \log(1/\alpha\gamma)\log(1/\alpha)}
    {\alpha^2 \tau^2}\right),\]
then so long as $e^{-\alpha^2} <\beta_1 < \beta/3  $, we can apply Theorem~\ref{thm:fatshatgen} to conclude that
\[\Pr_{S \sim D^n}\left[\Pr_{(x,y) \sim D}[H(x) \neq y] < \alpha \right] \geq 1 - \beta.\]
\end{proof}

\bibliography{main_biber}

\clearpage

\appendix

\clearpage
\section{Glossary of Symbols}
\label{sec:glossary-symbols}

\begin{description}
\item[$n$] Sample size
\item[$\alpha$] Final desired accuracy of output hypothesis, as in
  $(\alpha, \beta)$-PAC learning
\item[$\beta$] Probability of catastrophic learning failure, as in
  $(\alpha, \beta)$-PAC learning
\item[$\epsilon$] Final desired privacy of output hypothesis, as in
  $(\epsilon, \delta)$-DP
\item[$\delta$] Final desired ``approximation'' of privacy, as in
  $(\epsilon, \delta)$-DP
\item[$\eta$] Rate of random classification noise
\item[$\tau$] Margin of underlying target halfspace.
\item[$\zeta$] Slickness parameter.
\item[$d$] Dimension of the input examples from $\R^d = X$
\item[$\gamma$] Advantage of Weak Learner
\item[$\nu$] Gaussian random variable denoting noise added to Weak
  Learner
\item[$\sigma$] Magnitude of Gaussian noise added to Weak Learner
\item[$\kappa$] Density parameter of LazyBregBoost
\item[$\lambda$] Learning rate of Dense Multiplicative Weights
\item[$\theta$] Margin induced by ``Optimal'' Boosting (Pythia)
\item[$X$] A domain.
\item[$\cX$] Domain of examples, specifically.
\item[$\mu$] A bounded measure.
\item[$\hat{\mu}$] A normalized bounded measure; a distribution.
\item[$\tilde{\mu}$] A bounded measure that has \emph{not} been
  Bregman projected.
\item[$\Delta(Z,Y)$] Total variation distance between $Z$ and $Y$
\item[$M$] A two-player zero-sum game. Only in Appendices.
\end{description}

\clearpage

\section{Fuzzy Halfspaces have Bounded Fat-Shattering Dimension}
\label{sec:fat-shatter-apx}

We recall and prove Lemma \ref{lem:shattdim}.

\FuzzyFatHS*

This follows from the lemmas below due to Servedio, Bartlett, and Shawe-Taylor.

\begin{lemma}[\citet{Ser00}]\label{lem:shatt1}
If the set $\{x_1, \dots, x_n\}$ is $\gamma$-shattered by $\mathcal{H} = \{f(x) = z \cdot x : \|z\|_2 \leq 2\}$, then every $b \in \{-1,1\}^n$ satisfies
\[\big\| \sum_{i=1}^n b_i x_i \big\|_2 \geq \gamma n/2. \]
\end{lemma}

\begin{lemma}[\citet{Bartlett98}]\label{lem:shatt2}
For any set $\{x_1, \dots, x_n\}$ with each $x_i \in \R^d$ and $\|x_i\|_2 \leq 1$, then there is some $b \in \{-1,1\}^n$ such that $\big\|\sum_{i=1}^n b_i x_i\|_2 \leq \sqrt{n}$.
\end{lemma}

\begin{proof}
We begin by showing that, with high probability, the hypothesis output by Algorithm~\ref{alg:HS-StL} is in the class $\mathcal{H} = \{f(x) = z \cdot x : \|z\|_2 \leq 2\}$. To bound the $\ell_2$ norm of $z$, we observe that 
\begin{align*}
\|z\|_2 &= \frac{1}{T}\big\|\sum_{t=1}^T \hat{z}_t \big\|_2 \leq \frac{1}{T}\sum_{t=1}^T \big\|\sum_{i=1}^n \normmu_t(i) y_i x_i \big\|_2 + \|\nu_t\|_2 = 1 + \frac{1}{T}\sum_{t=1}^T \|\nu_t\|_2  
\end{align*}
where $\hat{z}_t$ denotes the weak learner hypothesis at round $t$ of boosting, and $\nu_t$ denotes the Gaussian vector added to the hypothesis at round $t$. 
Letting $\sigma =\tau/8c\sqrt{\log\left(\frac{3072\log(1/\kappa)}{\beta \tau^2}\right)}$ for a constant $c$, it follows that with probability at least $1 - \frac{\beta \tau^2}{3072\log(1/\kappa)} = 1 - \tfrac{\beta}{3T}$, a given $\nu_t$ has $\|\nu_t\|_2 \leq \tau/8 < 1$, and therefore with probability $(1 -\tfrac{\beta}{3T})^T \geq (1 - \tfrac{\beta}{3}) $, $\frac{1}{T}\sum_{t=1}^T \|\nu_t\|_2 < 1$, and so $\|z\|_2 \leq 2$. Therefore with all but probability $\beta/3$, the hypothesis output by Algorithm~\ref{alg:HS-StL} is in the class $\mathcal{H}$.

From Lemma~\ref{lem:shatt1}, it cannot be the case that a set $\{x_1, \dots, x_n \}$ is $\gamma$-shattered by $\mathcal{H}$ if there exists a $b \in \{-1,1\}^n$ such that 
\[\big\| \sum_{i=1}^n b_i x_i \big\|_2 < \gamma n/2.\]
At the same time, it follows from Lemma~\ref{lem:shatt2} that if $n > 4/\gamma^2$, such a $b \in \{-1,1\}^n$ must exist. Therefore the fat-shattering dimension of $\mathcal{H}$ at margin $\gamma$ is $fat_{\mathcal{H}}(\gamma) \leq 4/\gamma^2$. Since our final hypothesis is in $\mathcal{H}$ with probability $1-\beta/3$, our claim holds. 
\end{proof}

\section{Privacy-Only Noise-Tolerant Sample Bound for Large-Margin Halfspaces}
\label{sec:proof-noise-tolerant}

We state and prove the formal version of Theorem \ref{ithm:privacy-only-HS-learn}.

\begin{theorem}[Learning Halfspaces Under Random Label Noise]
  \label{thm:st-hs-learn-noise-apx}
  The \texttt{HS-StL} procedure is a
  $(\epsilon, \delta)$-Differentially Private $(\alpha, \beta)$-strong
  PAC learner for $\tau$-margin halfspaces tolerating random label
  noise at rate $\eta = O(\alpha \tau)$ with sample complexity
  \begin{equation*}
    n = \tilde{\Omega} \Bigg(
      \underbrace{
          \frac{1}{\epsilon \alpha \tau^2}
        }_{\substack{\text{privacy}
            \\\ \text{(Claim \ref{claim:enough-samples-to-privacy}})}}
      +
      \underbrace{
          \frac{1}{\alpha^2 \tau^2}
        }_{\text{accuracy}}
      +
      \underbrace{
        \frac{1}{ \epsilon^2}
        +
        \frac{1}{ \alpha^2}
      }_{\substack{\text{generalization}
        \\\ \text{(Claim \ref{claim:enough-samples-to-gen})}}}
    \Bigg)
  \end{equation*}
\end{theorem}

\begin{proof}
  Denote by $\epsilon_T$ and $\delta_T$ the parameters of approximate
  differential privacy at the final round $T$ of \texttt{HS-StL}, and
  by the $H$ the output hypothesis of \texttt{HS-StL}. We proceed as
  follows.

  \begin{enumerate}
  \item Given enough samples, \texttt{HS-StL} is differentially
    private. (Claim \ref{claim:enough-samples-to-privacy})
  \item Random Label Noise at rate $\eta = O(\alpha\tau)$ will
    (w.h.p.) not ruin the sample. (Claim
    \ref{claim:hopeless-samples-unlikely})
  \item The Gaussian Mechanism will (w.h.p.) not ruin the weak
    learner. (Claim \ref{claim:bad-gauss-unlikely})
  \item Given enough samples, training error is (w.h.p.) close to test
    error. (Claim \ref{claim:enough-samples-to-gen})
  \item Given enough samples, \texttt{HS-StL} (w.h.p.) builds a
    hypothesis with low test error. 
  \end{enumerate}

  For the remainder of this proof, fix the settings of all parameters
  as depicted in \texttt{HS-StL} (Algorithm \ref{alg:HS-StL}). We
  reproduce them here:
  \begin{gather}
    \kappa \gets \alpha/4\label{eq:1} \\
    \sigma \gets \tau/8c\sqrt{\log\left(\frac{3072\log(1/\kappa)}{\beta \tau^2}\right)}
  \end{gather}

  \begin{claim}[Enough Samples $\implies$ \texttt{HS-StL} is
    Differentially Private]
    \label{claim:enough-samples-to-privacy}
    For every $\delta_T > 0$, we have:
    \[
      n > \tilde{O}\left( \frac{1}{\epsilon_T \alpha \tau^2} \right)
      \implies
      \text{\texttt{HS-StL}  is }(\epsilon_T, \delta_T)\text{-DP}
    \]
    
  \end{claim}
  \begin{proof}
    By the privacy bound for $\widehat{\weak}$ (Theorem
    \ref{thm:wL-zCDP}), we know that a single iteration of
    \texttt{Boost} is $\rho$-zCDP for
    $\rho = \frac{8}{(\kappa n \sigma)^2}$. Then, \texttt{Boost} runs
    for $T = \frac{1024\log(1/\kappa)}{\tau^2}$ rounds. So, by tight
    composition for zCDP (Lemma \ref{lem:zcdp-compose}),
    \texttt{HS-StL} is $\rho_T$-zCDP where:
    \[
      \rho_T = O \left(
        \frac{\log(1/\kappa)}{(\kappa n \sigma \tau)^2}
      \right)
    \]
    Now we convert from zero-concentrated to approximate differential
    privacy, via Lemma \ref{lem:zcdp-to-adp}: if
    $\epsilon_T < 3\sqrt{\rho_T\log(1/\delta_T)}$, then
    \texttt{HS-StL} is $(\epsilon_T, \delta_T)$-DP for all
    $\delta_T > 0$. We re-arrange to bound $n$.
    
    \[\rho_T < O \left( \frac{\epsilon_T^2}{\log(1/\delta_T)} \right) \]

    Unpacking $\rho_T$ we get:
    \[
      \frac{\log(1/\kappa)}{(\kappa n \sigma \tau)^2} <
      O \left( \frac{\epsilon_T^2}{\log(1/\delta_T)} \right)
    \]

    This will hold so long as:

    \[
      n > \Omega \left(
        \frac{\sqrt{\log(1/\kappa)\log(1/\delta_T)}}
        {\kappa \sigma \tau \epsilon_T}
      \right)
    \]
    
    Substituting the settings of $\sigma$ and $\kappa$ from
    \texttt{HS-StL}, we obtain:

    \[n > \Omega \left(
        \frac{\sqrt{\log(1/\alpha)\log(1/\delta_T)\log(\log(1/\alpha)/\beta\tau^2)}}
        {\epsilon_T \alpha \tau^2}
      \right)
    \]

  \end{proof}
  
  We next consider the two events that could destroy good training
  error.
  
  \begin{description}
  \item[Too Many Corrupted Samples] Noise could corrupt so many
    samples that the weak learner fails. Under an approprite noise
    rate, this is unlikely. We denote this event by \textsf{BN}
    (for ``bad noise''). 
  \item[Gaussian Mechanism Destroys Utility] The Gaussian noise
    injected to ensure privacy could destroy utility for a round of
    boosting. We denote this event by \textsf{BG} (for ``bad
    Gaussian''). 
  \end{description}

  Both events are unlikely, under the settings of \texttt{HS-StL}.
  
  \begin{claim}[Hopelessly Corrupted Samples are Unlikely]
    \label{claim:hopeless-samples-unlikely}
    Let $\mathsf{F}_1, \dots, \mathsf{F}_n$ indicate the event ``label
    $i$ was flipped by noise,'' and denote by
    $\mathsf{F} = \sum_{i=1}^n \mathsf{F}_i$ the number of such
    corrupted examples. Under the settings of \texttt{HS-StL} and
    noise rate $\eta = \alpha\tau /32$, we have:
    \[
      n > \frac{96 \ln(4/\beta)}{\alpha \tau}
      \implies
      \Pr[\mathsf{BN}] = \Pr[ \mathsf{F} > \kappa n ] \leq \beta / 4
    \]
  \end{claim}

  \begin{proof}
    At noise rate $\eta$, we have $\ex{\rv{F}} = n\eta$. From the
    definitions and Theorem \ref{thm:wL-zCDP},
    \[
      \Pr[\textsf{BN}] = \Pr \left[
        \mathsf{F} \geq \frac{\alpha\tau}{16} n
      \right]
    \]

    We apply the following simple Chernoff bound:
    $\forall \delta \geq 1$
    \[
      \prob{\rv{F} \geq (1 + \delta) \ex{\rv{F}}}
      \leq \exp(-\ex{ \rv{F} } \delta / 3)
    \]
    Substituting with $\delta = 1$:
    \[
      \prob{\rv{F} \geq 2\eta n} \leq \exp(-(\eta n)/ 3)
    \]

    Noise rate $\eta = \alpha\tau/32$ gives the appropriate event
    above:

    \[
      \prob{\rv{F} \geq 2\eta n} = \prob{\rv{F} \geq \frac{\alpha\tau}{16} n}
      \leq
      \exp(-( \alpha\tau n) / 96)
    \]

    Constraining the above probability to less than $\beta / k_1$
    for any constant $k_1 > 1 $ we solve to obtain:

    \[
      n > \frac{96 \ln(k_1/\beta)}{\alpha \tau}
    \]
    
  \end{proof}

  \begin{claim}[Bad Gaussians are Unlikely]
    \label{claim:bad-gauss-unlikely}
    Let $\textsf{BG}_i$ indicate the event that the $i$th call to the
    weak learner fails to have advantage at least $\tau/8$. Under the
    settings of \texttt{HS-StL}:
    \[
      \Pr[\mathsf{BG}] = \Pr[ \exists i ~\mathsf{BG}_i] \leq \beta / 2
    \]
  \end{claim}
  \begin{proof}
    Our setting of $\sigma$ simplifies the advantage of
    $\widehat{\weak}$ to $\gamma(\tau, \sigma) = \tau/8$ with all but
    probability $\xi = \frac{\beta \tau^2}{1024\log(1/\kappa)}$. Then,
    by the round bound for \texttt{LB-NxM} (Theorem
    \ref{thm:BregBoost-RoundBound}), \texttt{Boost} will terminate
    after
    $T = \frac{8\log(1/\kappa)}{\gamma^2} =
    \frac{512\log(1/\kappa)}{\tau^2}$ rounds, and so we have that with
    probability $(1 - \xi)^T \geq 1 - \frac{\beta}{2}$ every
    hypothesis output by $\widehat{\weak}$ satisfies the advantage
    bound.

  \end{proof}

  To enforce generalization, we capture both training and test error
  for any hypothesis $H$ with a statistical query that indicates
  misclassification. Evaluated over the sample it is the training
  error of $H$, and evaluated over the population it is the test error
  of $H$.

  \[
    \err_H(x, y) \mapsto \begin{cases}
      1 &\text{if } yH(x) \leq 0 \\
      0 &\text{otherwise}
    \end{cases}
  \]
  
  \begin{claim}[Enough Samples $\implies$ Good Generalization]
    \label{claim:enough-samples-to-gen}
    If ~$0 < \epsilon_T < 1/3$ and $0 < \delta_T < \epsilon_T/4$
    \[
      n \geq \tilde{\Omega} \left( \frac{1}{\epsilon_T^2} \right)
      \implies
      \Prob{S \sim D^n}{ \err_H(D) \geq \err_H(S) + 18 \epsilon_T }
      \leq \delta_T / \epsilon_T
    \]
  \end{claim}

  \begin{proof}
    Define (for analysis only) a procedure \texttt{HS-StL-test}, which
    outputs the ``error'' statistical query. That is, letting $H = $
    \texttt{HS-StL}$(S)$ where $S \sim D^n$, \texttt{HS-StL-test}
    prints $\err_H$. Thus, \texttt{HS-StL-test} is a mechanism for
    selecting a statistical query.
    
    Because \texttt{HS-StL-test} is simple post-processing, it
    inherits the privacy of \texttt{HS-StL}. Since we select $\err_H$
    privately, it will (by Theorem \ref{thm:dp-generalization}) be
    ``similar'' on the sample $S$ (training error) and the population
    $D$ (test error).  Ignoring over-estimates of test error and
    observing the sample bounds of Theorem
    \ref{thm:dp-generalization}, this gives the claim.

  \end{proof}
  
  \begin{claim}[Low Training Error is Likely]
    \label{claim:good-train-likely}
    Given $\neg\rv{GB}$ and $\neg{\rv{BN}}$, we have
    ~\( \err_H(S) \leq \alpha/2\).
  \end{claim}

  \begin{proof}
    Let $\tilde{S}$ denote the noised sample, and let $S_C$ and $S_D$
    be the ``clean'' and ``dirty'' subsets of examples,
    respectively.
    
    Given $\neg\rv{GB}$ and $\neg{\rv{BN}}$, the weak learning
    assumption holds on every round. So, by Theorem
    \ref{thm:BregBoost-RoundBound}, the boosting algorithm will attain
    low training error $\err_H(\tilde{S}) = \kappa$. This is not yet
    enough to imply low test error, because
    $\tilde{S} \not\sim_{iid} D$. So we bound $\err_H(S)$ using
    $\err_H(\tilde{S})$. Suppose that the noise affects training in
    the worst possible way: $H$ fits \emph{every} flipped label, so
    $H$ gets \emph{every} example in $S_D$ wrong. Decompose and bound
    $\err_H(S)$ as follows:

    \begin{align*}
      \err_H(S) &= \sum_{(x,y) \in S} \err_H(x,y) \\
             &= \sum_{(x,y) \in S_C} \err_H(x,y) + \sum_{(x,y) \in S_D} \err_H(x,y) \\
             &\leq \kappa + |S_D| &\text{Boosting Theorem, worst case fit} \\
             &\leq \kappa + \frac{\alpha\tau}{16} &\neg\rv{BN}
    \end{align*}
    
    Because the sample is from the unit ball, we have $\tau \in
    (0,1)$. Therefore, it is always the case that
    $\frac{\alpha\tau}{16} < \frac{\alpha}{4}$. So
    $\err_H(S) \leq \kappa + \alpha/4 \leq \alpha/2$, concluding proof of
    the above claim.
  \end{proof}
  
  It remains only to select $\epsilon_T$ and $\delta_T$ so that the
  claims above may be combined to conclude low test error with high
  probability. Recall that our objective is sufficient privacy to
  simultaneously:
  
  \begin{enumerate}
  \item Ensure that \texttt{HS-StL} is $(\epsilon, \delta)$-DP
  \item Apply DP to generalization transfer (Theorem
    \ref{thm:dp-generalization}) for good test error.
  \end{enumerate}

  Both these objectives impose constraints on $\epsilon_T$ and
  $\delta_T$.  The requirement that the algorithm is desired to be
  $(\epsilon, \delta)$-DP in particular forces $\epsilon_T$ to be
  smaller than $\epsilon$ and $\delta_T$ to be smaller than
  $\delta$. The transfer theorem is slightly more subtle; to
  PAC-learn, we require:
  
  \[
    \Pr_{S \sim D^n}[\err_H(D) \geq \alpha] \leq \beta
  \]

  While Claim \ref{claim:enough-samples-to-gen} gives us that

  \[
    \Prob{S \sim D^n}{ \err_H(D) \geq \err_H(S) + 18 \epsilon_T } \leq
    \delta_T / \epsilon_T
  \]

  So, accounting for both privacy and accuracy, we need
  $\epsilon_T < \phi = \min(\epsilon, \alpha/36)$. We can select any
  $\delta_T < \min(\delta, ~\phi\beta/4)$ to ensure that
  $\delta_T/\epsilon_T < \beta/2$. By substituting the different
  realizations of these '$\min$' operations into Claims
  \ref{claim:enough-samples-to-privacy} and
  \ref{claim:enough-samples-to-gen} we obtain the sample bound.

  Finally, observe that with these settings we can union bound the
  probability of $\rv{BN}$ and $\rv{BG}$ and the event that
  generalization fails with $\beta$, as required for PAC learning. But
  it follows from the claims above that if $\neg\rv{BN}$ and
  $\neg{\rv{BG}}$ and a good transfer all occur, then the ouput
  hypothesis $H$ has test error less than $\alpha$, concluding the
  argument.
  
\end{proof}

\section{Smooth Boosting via Games}
\label{sec:boost-via-games}
Here, we prove our round-bound and final margin guarantee for
\texttt{LazyBregBoost}. The proof is a reduction to approximately
solving two-player zero-sum games. We introduce the basic elements of
game theory, then outline and excute the reduction. Overall, we recall
and prove Theorem \ref{thm:BregBoost-RoundBound}:

\LazyBregRB*

\subsection{Two-Player Zero-Sum Games}
A two player game can be described by a matrix, where the \emph{rows}
are indexed by ``row player'' strategies $\cP$, the \emph{columns} are
indexed by ``column player'' strategies $\cQ$, and each entry $(i,j)$
of the matrix is the \emph{loss} suffered by the row player when row
strategy $i \in \cP$ is played against column strategy $j \in
\cQ$. Such a game is \emph{zero-sum} when the colum player is given as
a reward the row player's loss. Accordingly, the row player should
minimize and the column player should maximize.

A single column or row is called a \emph{pure strategy.} To model
Boosting, we imagine players who can randomize their actions. So the
fundamental objects are \emph{mixed strategies:} distributions $P$
over the rows and $Q$ over the columns. Playing ``according to'' a
mixed strategy means sampling from the distribution over pure
strategies and playing the result. When two mixed strategies are
played against each other repeatedly, we can compute the
\emph{expected loss} of $P$ vs.  $Q$ playing the game $M$:

\begin{equation}
  \label{eq:loss}
  M(P,Q) = \underbrace{\sum_{i,j \in \cP \times \cQ} P(i) M(i,j) Q(j)}_{\text{(i)}}
         = \underbrace{\sum_{j \in \cQ} M(P,j)Q(j)}_{\text{(ii)}}
         = \underbrace{\sum_{i \in \cP} P(i) M(i,Q)}_{\text{(iii)}}
\end{equation}

``Iterated play'' pits the row player against an arbitrary environment
represented by the column player. At each round, both the row player
and column player choose strategies $P_t$ and $Q_t$ respectively. The
expected loss of playing $Q_t$ against each \emph{pure} row strategy
is revealed to the row player. Then, the row player suffers the
\emph{expected loss} of $P_t$ vs. $Q_t$. This set-up is depicted by
Algorithm \ref{alg:iter-play}. Good row player strategies have
\emph{provably bounded regret} --- they do not suffer much more loss
than the best possible \emph{fixed} row player strategy in hindsight
during iterated play.

\begin{algorithm}
  \caption{Iterated Play}
  [\emph{\textbf{Input:}} $T$ the number of rounds to play for \\
  \emph{\textbf{Output:}} Total expected row player cost incurred]
  \label{alg:iter-play}
  \begin{algorithmic}
    \FOR{$t = 1$ \TO $T$}
    \STATE $P_t \gets $ Row player choice of mixed strategies, seeing $\ell_1, \dots, \ell_{t-1}$
    \STATE $Q_t \gets $ Column player choice of mixed strategies, seeing $P_t$
    \STATE $\ell_t(i) \gets M(i,Q_t) ~ \forall i$
    \COMMENT{Reveal loss on each pure row strategy}
    \STATE $C \gets C + M(P_t, Q_t)$
    \COMMENT{Accumulate total loss}
    \ENDFOR
  \end{algorithmic}
\end{algorithm}

Here, we reduce boosting to the ``Lazy Dense Multiplicative Updates''
row player strategy (Algorithm \ref{alg:lazy-DUR}) which enjoys
bounded regret (Lemma \ref{lem:simple-regrets}, proved in Appendix
\ref{sec:lazy-regret-bound} for completeness) and two other helpful
properties:

\begin{description}
\item[Simple State:] The only state is all previous loss vectors and
  step count so far; this enables privacy.
\item[Single Projection:] It Bregman-projects just once per round;
  this enforces slickness.
\end{description}

\begin{algorithm}
  \caption{Lazy Dense Update Strategy (LDU)}\label{alg:lazy-DUR}
  \emph{\textbf{Input:}} $\mathcal{P}$, a set of pure row-player strategies,
  learning rate $\lambda$, losses $\ell_1, \dots, \ell_T$\\
  \emph{\textbf{Output:}} A measure over $\mathcal{P}$
  \begin{algorithmic}
  	\FOR{$i \in \mathcal{P}$} 
  	  	\STATE $\mu_1(i) \gets \kappa$
  	\ENDFOR 
    \FOR{$i \in \mathcal{P}$}
    	\STATE $\tilde{\mu}_{T+1}(i) \gets e^{-\lambda\sum_{t=1}^T \ell_t(i)} \mu_1(i)$
    \ENDFOR
    \STATE $\mu_{T+1} \gets \Pi_\Gamma \tilde{\mu}_T$
  \end{algorithmic}
\end{algorithm}

\begin{lemma}[Lazy Dense Updates Regret Bound]
  \label{lem:simple-regrets}
  Let $\Gamma$ be the set of $\kappa$-dense measures. Set $\mu_1(i) = \kappa$ for every $i$. Then for all $\mu \in \Gamma$ we have the following
  regret bound.
  
  \[\frac{1}{T}\sum_{t=1}^T M(\dst{\mu}_t, Q_t) \leq
    \frac{1}{T}\sum_{t=1}^T M(\dst{\mu}, Q_t)
    + \lambda
    + \frac{\KL{\mu}{\mu_1}}{\lambda \kappa |\cP| T} \]
\end{lemma}

\subsection{Reducing Boosting to a Game}
We present a reduction from boosting real-valued weak learners to
approximately solving iterated games, following \citet{FS94}. To prove
the necessary round-bound (Theorem \ref{thm:BregBoost-RoundBound}), we
do the following:

\begin{enumerate}
\item\textbf{Create a Game.} The meaning of ``advantage'' given by a
  Weak Learning assumption (Definition \ref{def:weaklearn}) naturally induces a
  zero-sum game where pure row strategies are points in the sample $S$
  and pure column strategies are hypotheses in $\cH$. The Booster will
  play mixed row strategies by weighting the sample and the weak
  learner will play pure column strategies by returning a single
  hypothesis at each round.
\item\textbf{Weak Learning $\implies$ Booster Loss Lower-Bound.} Weak
  Learners have some advantage in predicting with respect to
  \emph{any} distribution on the sample. Thus, the particular sequence
  of distributions played by \emph{any} Booster must incur at least
  some loss.
\item \textbf{Imagine Pythia, a Prophetic Booster.} Given perfect
  foreknowledge of how the weak learner will play, what is the best
  Booster strategy?  Create a ``prescient'' Boosting strategy
  $P^\star$ which concentrates measure on the ``worst'' examples
  $(x,y) \in S$ for the \emph{final} hypothesis $H$ at each round.
\item\textbf{How Well Does Pythia Play?} Upper-bound the total loss
  suffered by a Booster playing $P^\star$ each round.
\item\textbf{Solve for $T$:} Recall that we want the Booster to
  \emph{lose}. That is, we want an accurate ensemble of
  hypotheses. Combining the upper and lower bounds above with the
  regret bound, we solve for a number of rounds to ``play'' (Boost)
  for such that the size of the set of ``worst'' examples shrinks to a
  tolerable fraction of the sample. This gives the round-bound
  (Theorem \ref{thm:BregBoost-RoundBound}).
\end{enumerate}

\subsubsection{Create a Game}
\label{sec:soft-pun}
Our game is a continuous variant of the Boolean ``mistake matrix''
\citep{FS94}. Let $\cH \subseteq \{B_2(1) \to [-1,1]\}$ be a set of
bounded $\R$-valued hypothesis functions on the unit ball. These
functions will be the pure column player strategies. Now let
$S = (x_1,y_1), \dots , (x_n,y_n)$ be a set of points where
$y_i \in \{ \pm 1 \}$ and $x_i \in B_2(1)$. These points will be the
pure row player strategies. Having chosen the strategies, all that
remains is to define the entries of a matrix. To cohere with the
definition of weak learning for real-valued functions, we define the
following game of \emph{soft punishments}:

\[
  M^{\cH}_{S} := M^{\cH}_{S}(i, h)
  = 1 - \frac{1}{2}|h(x_i) - y_i|
\]

Notice the quantification here: we can define the soft punishment game
for any sample and any set of hypotheses. We will omit $\cH$ and $S$
when fixed by context. This is a simple generalization of the game
from \citet{FS94} which assigns positive punishment 1 to correct
responses, and 0 to incorrect responses. Since we work with
real-valued instead of Boolean predictors, we alter the game to scale
row player loss by how ``confident'' the hypothesis is on an input
point.

\subsubsection{Weak Learning $\implies$ Booster Loss Lower-Bound.}
Mixed strategies for the booster (row player) are just distributions
over the sample. So the accuracy assumption on the weak learner
$\mathtt{WkL}$, which guarantees advantage on \emph{every}
distribution, induces a lower bound on the total loss suffered by
\emph{any} booster playing against $\mathtt{WkL}$. Recall that losses
are measured with respect to \emph{distributions} over strategies, not
\emph{measures}. So, below we normalize any measure to a distribution
``just before'' calculating expected loss

\begin{lemma}[Utility of Weak Learning]
  \label{lem:wkl-implies-loss-lb}
  For any sequence of $T$ booster mixed strategies
  $(\mu_1, \dots, \mu_T)$, suppose the sequence of column point
  strategies $h_1, \dots, h_T$ is produced by a weak learner that has
  advantage $\gamma$. Then:
  \[
    \sum_{t=1}^T M(\dst{\mu}_t, h_t) \geq \sfrac{T}{2} +
    T\gamma
  \]
\end{lemma}

\begin{proof}
  \begin{align*}
    \sum_{t=1}^T M(\dst{\mu}_t, h_t)
    &= \sum_{t=1}^T \Ex{i \sim \dst{\mu}_t}{M(i,h_t)} &\text{ unroll def \ref{eq:loss} } (iii)\\
    &= \sum_{t=1}^T \Ex{i \sim \dst{\mu}_t}{1 - \sfrac{1}{2}|h_t(x_i) - y_i|}
      &\text{ re-arrange ``advantage''} \\
    &= \sum_{t=1}^T \sfrac{1}{2} + \Ex{i \sim \dst{\mu}_t}{\sfrac{1}{2} -\sfrac{1}{2}|h_t(x_i) - y_i|}
      &\text{linearity of }\mathbb{E}\\
    &= \sfrac{T}{2} + \sum_{t=1}^T \Ex{i \sim \dst{\mu}_t}{\sfrac{1}{2} h_t(x_i)y_i}
      &\text{ distributing summations}\\
    &\geq \sfrac{T}{2} + T\gamma &\text{ by Weak Learning Assumption}
  \end{align*}
\end{proof}

\subsubsection{Imagine Pythia, a Prophetic Booster.}

How should a booster play if she knows the future? Suppose Pythia
knows exactly which hypotheses $h_1, \dots, h_T$ the weak learner will
play, but is restricted to playing the same fixed $\kappa$-dense
strategy for all $T$ rounds. Intuitively, she should assign as much
mass as possible to points of $S$ where the combined hypothesis
$H= (1/T) \sum_{t \in [T]} h_t$ is incorrect, and then assign
remaining mass to points where $H$ is correct but uncertain. We refer
to this collection of points as $B$, the set of ``bad'' points for
$h_1, \dots, h_T$. We formalize this strategy as Algorithm
\ref{alg:pythia}.

\begin{algorithm}
  \caption{Pythia}\label{alg:pythia}
  \emph{\textbf{Input:}} $S$ a sample with $|S| = n$;
  $H$ a combined hypothesis;
  $\kappa$ a target density \\
  \emph{\textbf{Output:}} Distribution $P^{\star}$ over $[n]$;
  Minimum margin $\theta_T$
  \begin{algorithmic}
    \STATE{ $B \gets \{ i \in [1, n]  ~|~ y_iH(x_i) < 0 \}$ }
    \COMMENT{Place all mistakes in $B$}
    \STATE{Sort $ [1, n]\setminus B$ by margin of $H$ on each point}
    \STATE{$\theta_T \gets 0$}
    \WHILE{$|B| < \kappa n$}
    \STATE{Add minimum margin element $i$ of $[1, n] \setminus B$ to $B$}
    \STATE{Update $\theta_T$ to margin of $H$ on $(x_i, y_i)$}
    \ENDWHILE
    \STATE{$P^{\star} \gets $ the uniform distribution over $B$ }
    \STATE{Output $P^{\star}, \theta_T$}
  \end{algorithmic}
\end{algorithm}

The prophetic booster Pythia plays the uniform measure on a set $B$ of
``bad'' points selected by Algorithm \ref{alg:pythia}, normalized to a
distribution. That is:

\[
P^{\star}(i) = \begin{cases}
  1/|B| &\mbox{if } i \in B \\ 
  0 & \mbox{otherwise }
\end{cases}
\]

It is important to observe that if $i$ is outside of the ``bad set''
$B$, we know $H$ has ``good'' margin on $(x_i, y_i)$.  To quantify
this, observe that for all $i \in B$, $H$ has margin at most
$\theta_T$ on $(x_i,y_i)$.

\begin{proposition}[Bad Margin in Bad Set]
  \label{obs:bad-set-bad-margin}
  For every $i \in B$, we know $\sum_{t=1}^T y_ih(x_i) \leq T\theta_T$
\end{proposition}

\begin{proof}
  \begin{align*}
    i \in B \implies y_iH(x_i) &\leq \theta_T &\text{ inspection of Pythia, above} \\
    \frac{y_i}{T}\sum_{t=1}^Th_t(x_i)  &\leq \theta_T &\text{ unroll } H \\ 
    \sum_{t=1}^T y_ih_t(x_i)  &\leq T\theta_T &\text{ re-arrange } 
  \end{align*}
\end{proof}

\subsubsection{How Well Does Pythia Play?} Here, we calculate the
utility of foresight --- an upper-bound on the loss of 
$P^\star$. Suppose $H$ is the terminal hypothesis produced by the boosting algorithm. We
substitute $P^\star$ into the definition of expected loss for
$M^{\cH}_S$ (soft punishments) and relate the margin on the bad
set for $H$ to the cumulative loss of $H$, giving the following lemma.

\begin{lemma}[Excellence of Pythia]
  \label{lem:utility-pythia}
  Let $S$ be a sample, $(h_1, \dots, h_T) \in \mathcal{H}^T$ a sequence of hypotheses, $H = (1/T)\sum_{i=1}^T h_i$, and $\kappa \in [0, 1/2]$ a density parameter. Let
  $ P^\star, \theta_H  = \mathrm{Pythia}(S,H,\kappa)$. Then:
  \[
    \sum_{t=1}^T M(P^{\star}, h_t) \leq (T/2) + (T\theta_H)/2 
  \]
\end{lemma}

We require a simple fact about advantages. Since $h(x) \in [-1,+1]$
and $y \in \{\pm 1\}$, we know:

\begin{equation*}
  \begin{gathered}
    yh(x) = 1 - |h(x) - y| \\
    \implies    (1/2)(yh(x)) = (1/2) - (1/2)|h(x) - y|
  \end{gathered}
\end{equation*}

The entries of the soft punishments matrix can also be re-written by
formatting advantage as above. For $i \in [1,n]$ and $h \in \cH$ we
have:
\begin{equation}
  \label{eq:simple-spun-math}
  \spun(i, h) = (1/2) + (1/2)(y_ih(x_i))
\end{equation}

\begin{proof}
  We manipulate the total regret of $P^\star$ towards getting an
  upper-bound in terms of the minimum margin of $H$ and number of
  rounds played. 

  \begin{align*}
    \sum_{t=1}^T M(P^{\star}, h_t)
    &= \sum_{t=1}^T \sum_{i=1}^n
      P^{\star}(i) \cdot \spun(i,h_t)
    &\text{Part (iii) of Expected Loss (Definition \ref{eq:loss})}
    \\
    &= \sum_{t=1}^T \sum_{i\in B}
      P^{\star}(i) \cdot \spun(i,h_t)
    &\text{Restrict sum --- }P^{\star}(i) = 0 \text{ outside } B\\
    &= \frac{1}{|B|}\sum_{t=1}^T \sum_{i\in B}
      \spun(i,h_t)
    &\text{Factor out } P^{\star}(i) \text{ --- constant by definition}\\
    &= \frac{1}{|B|} \sum_{i\in B} \sum_{t=1}^T
      \left( (1/2) + (1/2)h_t(x_i)y_i \right)
    &\text{Equation \ref{eq:simple-spun-math} about } M_S^{\cH} \text{ entries}\\
    &= \frac{1}{|B|} \left(
      (|B|T)/2 +
      (1/2)\sum_{i \in B}\sum_{t=1}^Th_t(x_i)y_i \right)
    &\text{Algebra}\\
    &\leq \frac{1}{|B|} \left(
      (|B|T)/2 +
      (1/2)\sum_{i\in B}T\theta \right)
    &\text{Bad margin in } B
      \text{ (Proposition \ref{obs:bad-set-bad-margin})}\\
    &= (T/2) + (T\theta)/2 
    &\text{Evaluate \& re-arrange}
  \end{align*}
\end{proof}

\subsubsection{Solve for $T$.}
We now have an upper bound on the loss incurred by a prescient
booster, and a lower bound on the loss to \emph{any} booster under the weak
learning assumption. This allows us to ``sandwich'' the performance of
boosting according to the lazy dense updates (LDU, Algorithm
\ref{alg:lazy-DUR}) strategy between these two extremes, because LDU
has good performance relative to \emph{any} fixed strategy (Lemma
\ref{lem:simple-regrets}, proved in Appendix
\ref{sec:lazy-regret-bound}). This sandwich gives a relationship
between the number of rounds $T$ and the margin of the final
hypothesis, which we now solve for the number of rounds necessary to boost
using LDU to obtain a ``good'' margin on ``many'' samples.

\begin{lemma}\label{lem:roundbound}
  Let $S$ be a sample of size $n$, let $\mu_t$ be the measure produced at round $t$ by
  \texttt{NxM}($S, H_{t-1}$) playing the Lazy Dense Update Strategy of Algorithm~\ref{alg:lazy-DUR}, and let $h_t$
  be the hypothesis output by
  $\widehat{\mathtt{WkL}}(S, \dst{\mu}_{t-1}, \sigma)$ at round
  $t$. Then after $T \geq \frac{16\log(1/\kappa)}{\gamma^2}$ rounds of
  \texttt{Iter}, the hypothesis
  $H_T(x) = \frac{1}{T}\sum_{t=1}^{T}h_t(x)$ has margin at least
  $\gamma$ on all but $\kappa n$ many samples.
\end{lemma}
\begin{proof}
  Denote by $\cU(B)$ the uniform measure (all $x \in B$ assigned a
  weight of 1) on the bad set $B$ discovered by Pythia. Combining the
  regret bound comparing LDU to fixed Pythia with the lower bound on
  loss that comes from the weak learner assumption, we have, overall:
\[
  \frac{T}{2} + T\gamma
  \underbrace{\leq}_{\substack{\text{Weak Learning} \\\ \text{(Lemma \ref{lem:wkl-implies-loss-lb})}}}
  \sum_{t=1}^T \spun(\dst{\mu}_t, h_t)
  \underbrace{\leq}_{\substack{\text{Regret Bound} \\\ \text{(Lemma \ref{lem:simple-regrets})}}}
  \sum_{t=1}^T \spun(P^\star, h_t) + \lambda T
  + \frac{\KL{\cU(B)}{\mu_1}}{\kappa n \lambda }
\]
Apply Lemma
\ref{lem:utility-pythia}, replacing prescient Boosting play with the
upper bound on loss we obtained:
\[
  T\gamma \leq \sum_{t=1}^T \spun(\dst{\mu}_t, h_t)
  \leq \frac{T \theta_H}{2} + \lambda T
  + \frac{\KL{\cU(B)}{\mu_1}}{\kappa n \lambda}
\]
Let's compute the necessary KL-divergence, recalling that $|\cU(B)| = |\mu_1| = \kappa n$:
\begin{align*}
  \KL{\cU(B)}{\mu_1}
  &= \sum_{x \in B} \cU(B)(x) \log \left( \frac{\cU(B)(x)}{\mu_1(x)} \right) - |\cU(B)| + |\mu_1|\\
  &= \sum_{x \in B} \log \left( \frac{1}{\kappa} \right) \\
  &= \kappa n \log \left( \frac{1}{\kappa} \right)
\end{align*}
Substituting into the above and dividing through by $T$, we have:
\[
  \gamma \leq \frac{\theta_H}{2}  + \lambda + \frac{\log(1/\kappa)}{T \lambda}
\]
Which, setting
\[
  T = \frac{ 16\log\left( \sfrac{1}{\kappa} \right)}{\gamma^2}
  \text{ and } 
  \lambda = \gamma/4
\]
implies
\[
\theta_H \geq \gamma.
\]
This is the margin bound we claimed, for every $(x,y) \not\in B$.
\end{proof}

\section{Bounded Regret for Lazily Projected Updates}
\label{sec:lazy-regret-bound}
 
A ``lazy'' multiplicative weights strategy that, at each round,
projects only \emph{once} into the space of dense measures is
presented below. Here, we prove that this strategy (Algorithm
\ref{alg:lazy-DUR2}) has bounded regret relative to any fixed dense
strategy.

\begin{algorithm}
  \caption{Lazy Dense Update Process}\label{alg:lazy-DUR2}
  \emph{\textbf{Input:}} $\mathcal{P}$, a set of pure row-player strategies,
  learning rate $\lambda$, losses $\ell_1, \dots, \ell_T$\\
  \emph{\textbf{Output:}} A measure over $P$
  \begin{algorithmic}
  	\FOR{$x \in \mathcal{P}$} 
  	  	\STATE $\mu_1(x) \gets \kappa$
  	\ENDFOR 
    \FOR{$x \in \mathcal{P}$}
    	\STATE $\tilde{\mu}_{T+1}(x) \gets e^{-\lambda\sum_{t=1}^T \ell_t(x)} \mu_1(x)$
    \ENDFOR
    \STATE $\mu_{T+1} \gets \Pi_\Gamma \tilde{\mu}_T$
  \end{algorithmic}
\end{algorithm}

To analyze the regret of a row player playing the strategy of Algorithm~\ref{alg:lazy-DUR2}, we will need the following definition.

\begin{definition}[Strong convexity]
A differentiable function $f: \R^n \rightarrow \R$ is $\alpha$-strongly convex on $X \subset \R^n$ with respect to the $\ell_p$ norm if for all $x, y \in X$
\[\langle \nabla f(x) - \nabla f(y), x-y \rangle \geq \alpha \|x - y\|_p^2.\]
If $f$ is twice differentiable, then $f$ is $\alpha$-strongly convex on $X$ with respect to the $\ell_p$ norm if for all $x \in X$, $y \in \R^n$ 
\[y^T (\nabla^2 f(x)) y \geq \alpha \|y\|_p^2.\] 
\end{definition}
We now proceed to analyze Algorithm~\ref{alg:lazy-DUR2} by arguments closely following those found in Chapter 2 of Rakhlin's notes on online learning \citep{Rakhlin}. First, we show this update rule minimizes a sequential regularized
loss over \emph{all} dense measures.

\begin{lemma}\label{lem:opt}
	Let $M(\normmu, Q_t) = \mathbb{E}_{i \sim \normmu }[ \ell_t(i)]$. Let $\Gamma$ be the set of $\kappa$-dense measures, and let $\mu_1$ be the uniform measure over $\mathcal{P}$ of density $\kappa$ ($\mu_1(x) = \kappa$ for all $x \in P$). Then for all $T \geq 1$, the measure $\mu_{T+1}$ produced by Algorithm~\ref{alg:lazy-DUR2} satisfies
	\[
	\mu_{T+1} = \argmin_{\mu \in \Gamma} \left[ \lambda |\mu| \sum_{t=1}^T M(\normmu, Q_t) + \KL{\mu}{\mu_1} \right]
	\]
\end{lemma}
\begin{proof}
  Suppose towards contradiction that there exists $\mu \in \Gamma$
  such that
	\[\lambda |\mu| \sum_{t=1}^T M(\normmu, Q_t) + \KL{\mu}{\mu_1} < \lambda |\mu_{T+1}| \sum_{t=1}^T M(\normmu_{T+1}, Q_t) + \KL{\mu_{T+1}}{\mu_1}.\]
Then it must be the case that         
	\begin{align*}
	\KL{\mu_{T+1}}{\mu_1} - \KL{\mu}{\mu_1} 
	&> \lambda \sum_{t=1}^T \langle \mu - \mu_{T+1}, \ell_t \rangle \\
	&= \langle \mu - \mu_{T+1}, \sum_{t=1}^T \lambda \ell_t \rangle \\
	&= \sum_i (\mu(i) - \mu_{T+1}(i))\log\tfrac{\mu_1(i)}{\unconstmu_{T+1}(i)} \\
	&= \sum_i \mu(i)\left(\log \tfrac{\mu(i)}{\unconstmu_{T+1}} - \log \tfrac{\mu(i)}{\mu_1(i)} \right) - \sum_i \mu_{T+1}(i) \left(\log \tfrac{\mu_{T+1}(i)}{\unconstmu_{T+1}(i)} - \log \tfrac{\mu_{T+1}(i)}{\mu_1(i)} \right) \\
	&= \KL{\mu}{\unconstmu_{T+1}} - \KL{\mu}{\mu_{1}} - \KL{\mu_{T+1}}{\unconstmu_{T+1}} + \KL{\mu_{T+1}}{\mu_{1}}.
	\end{align*}
	 Under our assumption, then, it must be the case that $0 > \KL{\mu}{\unconstmu_{T+1}} - \KL{\mu_{T+1}}{\unconstmu_{T+1}}$, but $\mu_{T+1}$ was defined to be the $\kappa$-dense measure that minimized the KL divergence from $\unconstmu_{T+1}$, and so we have a contradiction.
\end{proof}

Using Lemma~\ref{lem:opt}, we can now show the following regret bound for a row player that, at each round, ``knows" the column strategy $Q_t$ that it will play against that round. 
\begin{lemma}\label{lem:be-the-leader}
	Let $\Gamma$ be the set of $\kappa$-dense measures. Let $\mu_1$ be the uniform measure over $\mathcal{P}$ of density $\kappa$ ($\mu_1(x) = \kappa$ for all $x \in \mathcal{P}$), and let $|\mathcal{P}| = n$. Then for all $\mu \in \Gamma$ we have that 
	\[\sum_{t=1}^T M(\normmu_{t+1}, Q_t) \leq \sum_{t=1}^T M(\normmu, Q_t) + \frac{\KL{\mu}{\mu_1}}{\lambda\kappa n}\] 
\end{lemma}
\begin{proof}
	For $T=0$, this follows immediately from the definition of $\mu_1$.
	
	Assume now that for any $\mu \in \Gamma$ that 
	\[\sum_{t=1}^{T-1} M(\normmu_{t+1}, Q_t) \leq \sum_{t=1}^{T-1} M(\normmu, Q_t) + \frac{\KL{\mu}{\mu_1}}{\lambda \kappa n},\]
	and in particular 
	\[\sum_{t=1}^{T-1} M(\normmu_{t+1}, Q_t) \leq \sum_{t=1}^{T-1} M(\normmu_{T+1}, Q_t) + \frac{\KL{\mu}{\mu_1}}{\lambda \kappa n}.\]
	It follows that
	\begin{align*}
	\sum_{t=1}^{T} M(\normmu_{t+1}, Q_t) 
	&= \sum_{t=1}^{T-1} M(\normmu_{t+1}, Q_t) + M(\normmu_{T+1}, Q_t) \\
	&= \sum_{t=1}^{T-1} M(\normmu_{t+1}, Q_t) + \sum_{t=1}^{T}M(\normmu_{T+1}, Q_t) - \sum_{t=1}^{T-1}M(\normmu_{T+1}, Q_t) \\
	&\leq \sum_{t=1}^{T}M(\normmu_{T+1}, Q_t) +  \frac{\KL{\mu_{T+1}}{\mu_1}}{\lambda \kappa n} \\
	&\leq \sum_{t=1}^{T}M(\normmu, Q_t) +  \frac{\KL{\mu}{\mu_1}}{\lambda \kappa n},
	\end{align*}
	for all $\mu \in \Gamma$, since $\mu_{T+1} = \argmin_{\mu\in\Gamma} [\lambda |\mu| \sum_{t=1}^T M(\normmu, Q_t) + \KL{\mu}{\mu_1}]$.
\end{proof}

To show a regret bound for the lazy dense update rule of Algorithm~\ref{alg:lazy-DUR2}, we now need only relate the lazy-dense row player's loss to the loss of the row player with foresight. To prove this relation, we will need the following lemma showing strong convexity of $R(\mu) = \KL{\mu}{\mu_1}$ on the set of $\kappa$-dense measures. 
\begin{lemma}\label{lem:stronglyconvex}
	The function $R(\mu) = \KL{\mu}{\mu_1}$ is $(1/\kappa n)$-strongly convex over the set of measures with density no more than $\kappa$, with respect to the $\ell_1$ norm.
\end{lemma}
\begin{proof}
	Let $\mu$ be a measure of density $d \leq \kappa$, and
	let $x = (\tfrac{1}{\mu(1)}, \dots, \tfrac{1}{\mu(n)})$. 
	The Hessian of $R(\mu)$ is $\nabla ^2 R(\mu ) = I_n x$. Therefore, for all $y \in \R^n$, 
	\begin{align*}
		y^T \nabla^2 R(\mu) y &= \sum_i \frac{y_i^2}{\mu(i)} & \\
		& = \frac{1}{|\mu|}\sum_{i}\mu(i) \sum_i\frac{y_i^2}{\mu(i)}&\text{ multiply by 1} \\
		& \geq \frac{1}{|\mu|}\left( \sum_i \sqrt{\mu(i)}\frac{y_i}{\sqrt{\mu(i)}}\right)^2 & \text{Cauchy-Schwarz}\\
		& = \frac{1}{|\mu|}\|y\|_1^2
	\end{align*}
	Therefore $R(\mu)$ is strongly convex on the given domain, for the given norm. 
\end{proof}

We can now relate the losses $M(\normmu_T, Q_t)$ and $M(\normmu_{T+1}, Q_t)$.
\begin{lemma}\label{lem:compare-to-future}
	Let $R(\mu) = \KL{\mu}{\mu_1}$, which is $1/\kappa n$-strongly convex with respect to the $\ell_1$ norm on $\Gamma$, the set of $\kappa$-dense measures. Then for all $T\geq 1$
	\[M(\normmu_T, Q_t) - M(\normmu_{T+1}, Q_t) \leq \lambda\]
\end{lemma}

\begin{proof}
	We first note that	$M(\normmu_T, Q_t) - M(\normmu_{T+1}, Q_t) = \frac{1}{\kappa n} \langle \mu_T - \mu_{T+1}, \ell_T  \rangle$. 
	So it suffices to show that 
	\[\sum_i (\mu_T(i) - \mu_{T+1}(i)) \ell_T(i) \leq \kappa n \lambda\]
	 which, because $\ell_T(i) \in [0,1]$, is implied by $\|\mu_t(i) - \mu_{t+1}(i)\|_1 \leq \kappa n \lambda$.
	 
	 Our strong convexity assumption on $R$ and an application of Bregman's theorem (Theorem \ref{thm:breg}) give us that
	 \begin{align*}
	 \tfrac{1}{\kappa n}\|\mu_T - \mu_{T+1} \|_1^2 
	 &\leq \langle \nabla R(\mu_T)- \nabla R(\mu_{T+1}), \mu_T - \mu_{T+1} \rangle \\
	 & = \KL{\mu_T}{\mu_{T+1}} +\KL{\mu_{T+1}}{\mu_T} \\
	 & \leq \KL{\mu_T}{\unconstmu_{T+1}} - \KL{\mu_{T+1}}{\unconstmu_{T+1}} + \KL{\mu_{T+1}}{\unconstmu_{T}} - \KL{\mu_{T}}{\unconstmu_{T}} \\
	 & = \sum_i \mu_T(i)\left(\log \frac{\mu_T(i)}{\unconstmu_{T+1}(i)} -\log \frac{\mu_T(i)}{\unconstmu_{T}(i)} \right) - \sum_i \mu_{T+1}(i)\left(\log \frac{\mu_{T+1}(i)}{\unconstmu_{T+1}(i)} - \log \frac{\mu_{T+1}(i)}{\unconstmu_{T}(i)}\right) \\
	 & = \sum_i \mu_T(i)\left(\log \frac{\unconstmu_T(i)}{\unconstmu_{T+1}(i)}\right) - \sum_i \mu_{T+1}(i)\left(\log \frac{\unconstmu_{T}(i)}{\unconstmu_{T+1}(i)}\right) \\
	 & = \langle \mu_T - \mu_{T+1}, \lambda \ell_{T+1} \rangle \\
	 & \leq \|\mu_T - \mu_{T+1} \|_{1}  \|\lambda \ell_{T+1}\|_{\infty} 
	 \end{align*} 
	 Therefore
	 \[\tfrac{1}{\kappa n}\|\mu_T - \mu_{T+1}\|_1 \leq \lambda \|\ell_{T+1}\|_{\infty} \leq \lambda\]
	 and so 
	 \[\|\mu_T - \mu_{T+1}\|_1 \leq \lambda \kappa n.\]
\end{proof}

We are now ready to prove the regret bound for the lazy dense update
process of Algorithm~\ref{alg:lazy-DUR2}, stated earlier in a
simplified form as Lemma \ref{lem:simple-regrets}.
\begin{lemma}\label{lem:regrets}
	Let $\Gamma$ be the set of $\kappa$-dense measures. Let $\mu_1$ be the uniform measure over $\mathcal{P}$ of density $\kappa$ ($\mu_1(x) = \kappa$ for all $x \in \mathcal{P}$), and let $|\mathcal{P}| = n$. Then for all $\mu \in \Gamma$ we have the following regret bound.
	\[\frac{1}{T}\sum_{t=1}^T M(\normmu_t, Q_t) \leq \frac{1}{T}\sum_{t=1}^T M(\normmu, Q_t)+ \lambda + \frac{\KL{\mu}{\mu_1}}{\lambda \kappa n T} \]
\end{lemma}

\begin{proof}
	From Lemma~\ref{lem:be-the-leader}, we have that 	
	\[\frac{1}{T}\sum_{t=1}^T M(\normmu_{t+1}, Q_t) \leq \frac{1}{T}\sum_{t=1}^T M(\normmu, Q_t) + \frac{\KL{\mu}{\mu_1}}{\lambda \kappa n T}\] 
	Adding $\frac{1}{T}\sum_{t=1}^T M(\normmu_t, Q_t)$ to both sides of the inequality and rearranging gives
	\[\frac{1}{T}\sum_{t=1}^T M(\normmu_t, Q_t) \leq \frac{1}{T}\sum_{t=1}^T M(\normmu, Q_t)+ \frac{1}{T}\sum_{t=1}^T \left(M(\normmu_t, Q_t) - M(\normmu_{t+1}, Q_t)\right)  + \frac{\KL{\mu}{\mu_1}}{\lambda \kappa n T}. \]
	We may then apply Lemma~\ref{lem:compare-to-future} to conclude
	\[\frac{1}{T}\sum_{t=1}^T M(\normmu_t, Q_t) \leq \frac{1}{T}\sum_{t=1}^T M(\normmu, Q_t)+ \lambda + \frac{\KL{\mu}{\mu_1}}{\lambda \kappa n T}.\]
\end{proof}

\end{document}